\documentclass[nohyperref]{article}

\usepackage{microtype}
\usepackage{graphicx}
\usepackage{caption}
\usepackage{subcaption}
\usepackage{booktabs} 

\usepackage{hyperref}

\usepackage[accepted]{icml2022}

\usepackage{amsmath}
\usepackage{amssymb}
\usepackage{mathtools}
\usepackage{amsthm}
\usepackage{enumitem}

\usepackage{multirow}
\usepackage{xspace}
\newcommand{\proj}{GSAT\xspace}

\usepackage[capitalize,noabbrev]{cleveref}
\theoremstyle{definition}
\newtheorem{theorem}{Theorem}[section]

\theoremstyle{remark}

\usepackage[textsize=tiny]{todonotes}

\icmltitlerunning{Interpretable and Generalizable Graph Learning via Stochastic Attention Mechanism}

\begin{document}

\twocolumn[
\icmltitle{Interpretable and Generalizable Graph Learning via \\ Stochastic Attention Mechanism}



\icmlsetsymbol{equal}{*}

\begin{icmlauthorlist}
\icmlauthor{Siqi Miao}{yyy}
\icmlauthor{Miaoyuan Liu}{xxx}
\icmlauthor{Pan Li}{yyy}
\end{icmlauthorlist}

\icmlaffiliation{yyy}{Department of Computer Science, Purdue University, West Lafayette, USA}

\icmlaffiliation{xxx}{Department of Physics and Astronomy, Purdue University, West Lafayette, USA}

\icmlcorrespondingauthor{Siqi Miao}{miao61@purdue.edu}
\icmlcorrespondingauthor{Pan Li}{panli@purdue.edu}

\icmlkeywords{Machine Learning, ICML}

\vskip 0.3in
]



\printAffiliationsAndNotice{}  

\begin{abstract}
Interpretable graph learning is in need as many scientific applications depend on learning models to collect insights from graph-structured data. Previous works mostly focused on using post-hoc approaches to interpret pre-trained models (graph neural networks in particular). They argue against inherently interpretable models because the good interpretability of these models is often at the cost of their prediction accuracy. However, those post-hoc methods often fail to provide stable interpretation and may extract features that are spuriously correlated with the task. In this work, we address these issues by proposing \emph{Graph Stochastic Attention} (\proj). Derived from the information bottleneck principle, \proj injects stochasticity to the attention weights to block the information from task-irrelevant graph components while learning stochasticity-reduced attention to select task-relevant subgraphs for interpretation. The selected subgraphs provably do not contain patterns that are spuriously correlated with the task under some assumptions.
Extensive experiments on eight datasets show that \proj outperforms the state-of-the-art methods by up to 20\%$\uparrow$ in interpretation AUC and 5\%$\uparrow$ in prediction accuracy. Our code is available at \url{https://github.com/Graph-COM/GSAT}.
\end{abstract}

\section{Introduction}
Graph learning models are widely used in science, such as physics~\cite{bapst2020unveiling} and biochemistry~\cite{jumper2021highly}. In many such disciplines, building more accurate predictive models is typically not the only goal. It is often more crucial for scientists to discover the patterns from the data that induce certain predictions~\cite{cranmer2020discovering}. For example, identifying the functional groups in a molecule that yield its certain properties may provide insights to guide further experiments~\cite{wencel2013c}.

Recently, graph neural networks (GNNs) have become almost the de fato graph learning models due to their great expressive power~\cite{kipf2016semi, xu2018powerful}. However, their expressivity is often built upon a highly non-linear entanglement of irregular graph features. So, it is often quite challenging to figure out the patterns in the data that GNNs use to make predictions.

Many works have been recently proposed to extract critical data patterns for the prediction by interpreting GNNs in post-hoc ways~\cite{ying2019gnnexplainer,yuan2020xgnn,vu2020pgm,luo2020parameterized,schlichtkrull2021interpreting,yuan2021explainability,lin2021generative,henderson2021improving}.
They work on a pre-trained model and propose different types of combinatorial search methods to detect the subgraphs of the input data that affect the model predictions the most.

In contrast to the above post-hoc methods, inherently interpretable models have been rarely investigated for graph learning tasks. There are two main concerns regarding such models. First, the prediction accuracy and inherent interpretability of a model often forms a trade-off~\cite{du2019techniques}. Practitioners may not allow sacrificing prediction accuracy for better interpretability. Second, the attention mechanism, a widely-used technique to provide inherent interpretability, often cannot provide faithful interpretation~\cite{lipton2018mythos}. The rationale of the attention mechanism is to learn weights for different features during the model training, and the rank of the learned weights can be interpreted as the importance of certain features~\cite{bahdanau2014neural, xu2015show}. However, recent extensive evaluations in NLP tasks \cite{serrano2019attention,jain2019attention,mohankumar2020towards} have shown that the attention may not weigh the features that dominate the model output more than other features. In particular, for graph learning tasks, the widely-used graph attention models~\cite{velivckovic2018graph,li2015gated} seem unable to provide any reliable interpretation of the data~\cite{ying2019gnnexplainer,yu2020graph}.

Along another line of research, invariant learning~\cite{pearl2016causal,arjovsky2019invariant,chang2020invariant,krueger21a} has been proposed to provide inherent interpretability and better generalizability. They argue that the models na\"{i}vely trained over biased data may risk capturing spurious correlations between the input environment features and the labels, and thus suffer from severe generalization issues. So, they propose to train models that align with the causal relations between the signal features and the labels. 
However, such training approaches to match causal relations typically have high computational complexity.  


\begin{figure}[t]
\begin{center}
\centerline{\includegraphics[trim={0cm 0cm 0.8cm 0cm},clip,width=1\columnwidth]{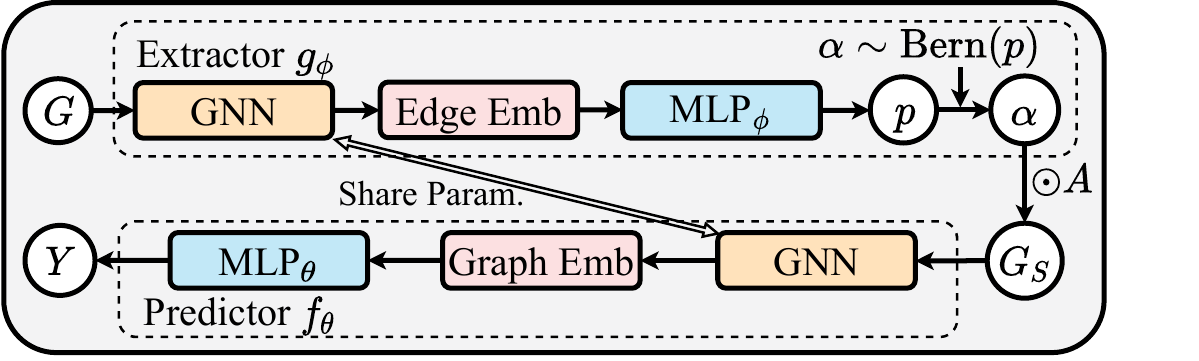}}
\vspace{-2mm}
\caption{The architecture of \proj. $g_{\phi}$ encodes the input graph $G$ and learns stochastic attention $\alpha$ (from Bernoulli distributions) that randomly drop the edges and obtain a perturbed graph $G_S$. $f_{\theta}$ encodes $G_S$ to make predictions. \proj does not constrain the size of $G_S$ but injects stochasticity to constrain information. The subgraph of $G_S$ with learnt reduced-stochasticity (edges with $p_e \rightarrow 1$) provides interpretation. \proj is a unified model by adopting just one GNN for both $g_{\phi}$ and $f_{\theta}$. \proj can be either trained from scratch or start from a pre-trained GNN predictor $f_{\theta}$.}
\label{fig:arch}
\end{center}
\vspace{-11mm}
\end{figure}

In this work, we are to address the above concerns by proposing \emph{Graph Stochastic Attention} (\proj), a novel attention mechanism to build inherently interpretable and well generalizable GNNs. The rationale of \proj roots in the notion of information bottleneck (IB)~\cite{tishby2000information,tishby2015deep}. We formulate the attention as an IB by injecting stochasticity into the attention to constrain the information flow from the input graph to the prediction~\cite{shannon1948mathematical}. Such stochasticity over the label-irrelevant graph components will be kept during the training while that over the label-relevant ones can automatically get reduced. This difference eventually provides model interpretation. By penalizing the amount of information from the input data, \proj is also expected to be more generalizable.

Our study achieves the following observations and contributions. First, the IB principle frees \proj from any potentially biased assumptions adopted in previous methods such as the size or the connectivity constraints on the detected graph patterns. Even when those assumptions are satisfied, \proj still works the best without using such assumptions, while when those assumptions are not satisfied, \proj achieves significantly better interpretation. See the sampled interpretation result visualizations in Fig.~\ref{fig:size-visual} and Fig.~\ref{fig:connect-visual}. Second, from the perspective of IB, all post-hoc interpretation methods are suboptimal. They essentially optimize a model without any information control and then perform a single-step projection to an information-controlled space, which makes the final interpretation performance sensitive to the pre-trained models. Third, by reducing the information from the input graph, \proj can provably remove spurious correlations in the training data under certain assumptions and achieve better generalization. Fourth, if a pre-trained model is provided, \proj may further improve both of its interpretation and prediction accuracy.

We evaluate \proj in terms of both interpretability and label-prediction performance. Experiments over 8 datasets show that \proj outperforms the state-of-the-art (SOTA) methods by up to $20\%$$\uparrow$  in interpretation AUC and $5\%$$\uparrow$ in prediction accuracy. Notably, \proj achieves the SOTA performance on \emph{molhiv} on OGB~\cite{hu2020ogb} among the models that do not use manually-designed expert features.

\begin{figure}[t]
\begin{center}
\centerline{\includegraphics[width=1.033\linewidth]{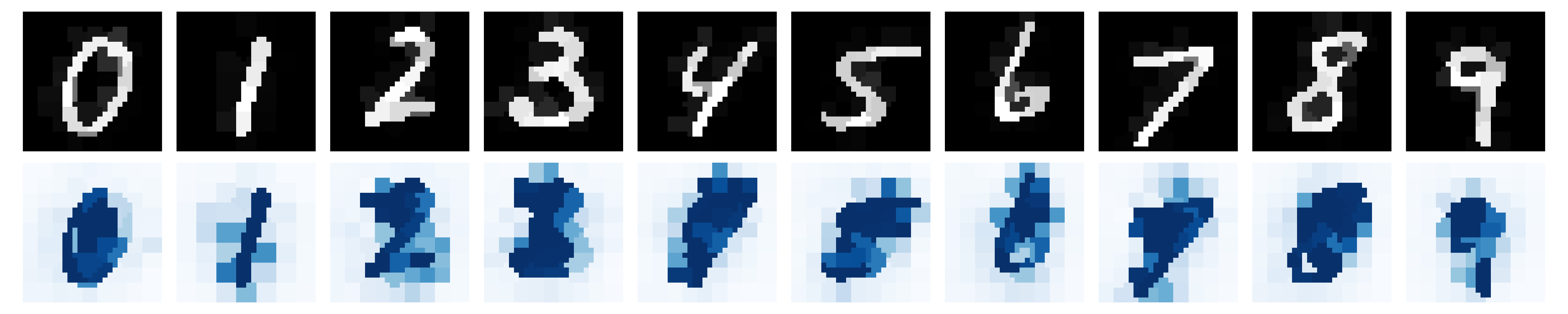}}
\vspace{-1mm}
\centerline{\includegraphics[width=1.05\linewidth]{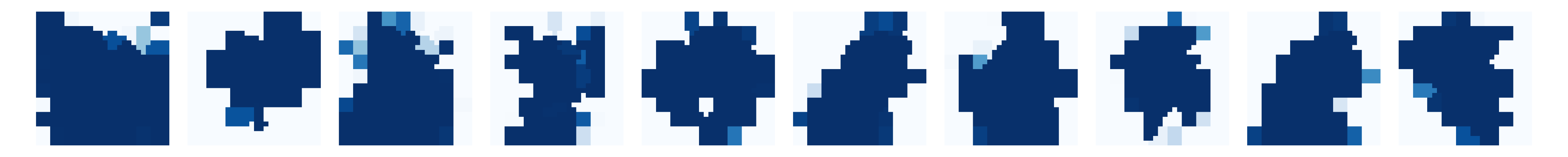}}
\end{center}
\vspace{-10mm}
\caption{Visualizing attention (normalized to $[0,1]$) of \proj (second row) v.s. masks of GraphMask~\cite{schlichtkrull2021interpreting} (third row) on MNIST-75sp. The first row shows the ground-truth. Different digit samples contain interpretable subgraphs of different sizes, while \proj is not sensitive to such varied sizes.}
\label{fig:size-visual}
\end{figure}
\begin{figure}[t]
\begin{center}
\vspace{-2mm}
\centerline{\includegraphics[trim={2cm 0cm 1cm 1.7cm},clip,width=1.0\linewidth]{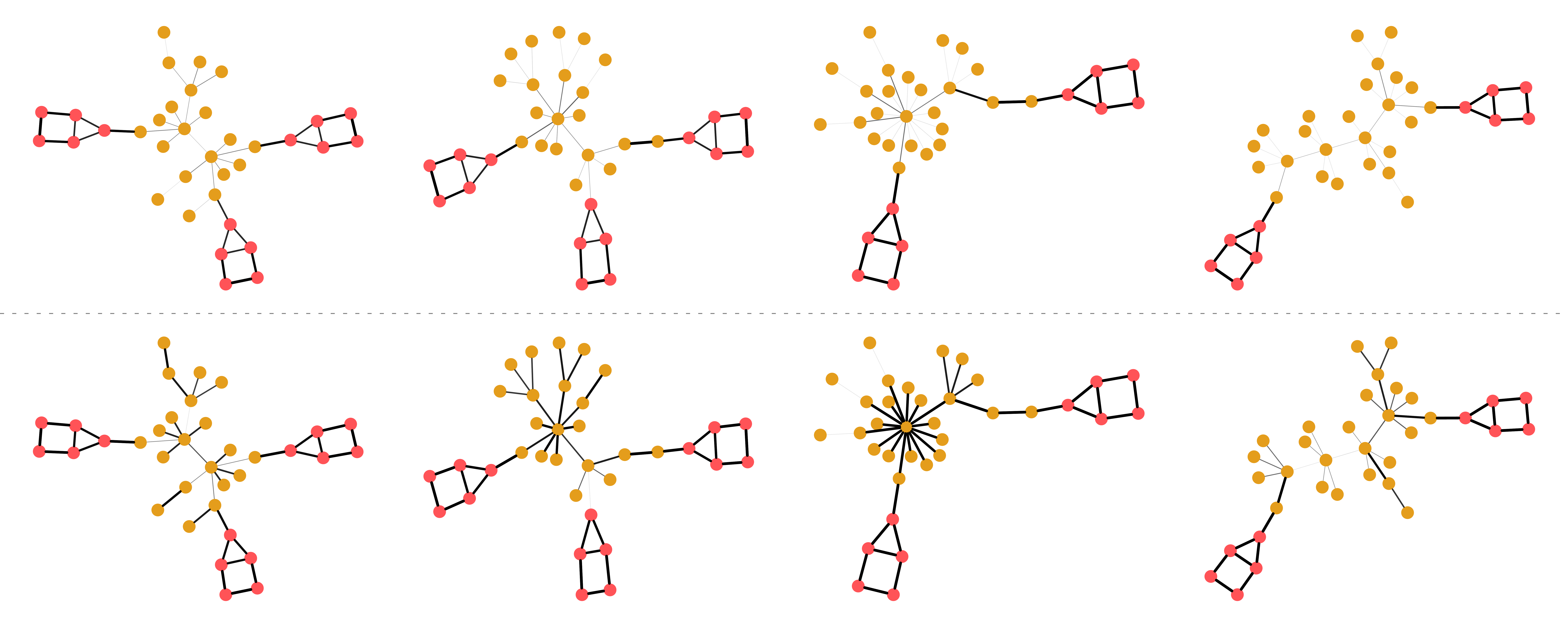}}
\end{center}
\vskip -10mm
\caption{Visualizing attention (normalized to $[0,1]$) of \proj (first row) and masks of GraphMask~\cite{schlichtkrull2021interpreting} (second row) on a motif example, where graphs with three house motifs and graphs with two house motifs represent two classes. Samples may contain disconnected interpretable subgraphs, while \proj detects them accurately. More details can be found in Appendix~\ref{appx:sup_exp}.}
\label{fig:connect-visual}
\vspace{-3mm}
\end{figure}

\section{Preliminaries}
As preliminaries, we define a few notations and concepts.

\textbf{Graph.} An attributed graph can be denoted as $G =(A,X)$ where $A$ is the adjacency matrix and $X$ includes node attributes. Let $V$ and $E$ denote the node set and the edge set, respectively. We focus on graph-level tasks: A training set of graphs with their labels $(G^{(i)}, Y^{(i)})$, $i=1,...,n$ are given, where each sample $(G^{(i)}, Y^{(i)})$ is assumed to be IID sampled from some unknown distribution $\mathbb{P}_{\mathcal{Y}\times\mathcal{G}}=\mathbb{P}_{\mathcal{Y}|\mathcal{G}}\mathbb{P}_{\mathcal{G}}$.

\textbf{Label-relevant Subgraph.} A label-relevant subgraph refers to the subgraph $G_S$ of the input graph $G$ that mostly indicates the label $Y$. For example, to determine the solubility of a molecule, the hydroxy group -OH is a positive-label-relevent subgraph, as if it exists, the molecule is often soluble to the water. Finding label-relevant subgraphs is a common goal of interpretable graph learning.

\textbf{Attention Mechanism.} 
Attention mechanism has been widely used in interpretable neural networks for NLP and CV tasks~\cite{bahdanau2014neural, xu2015show, vaswani2017attention}.
However, GNNs with attention~\cite{velivckovic2018graph} often generate low-fidelity attention weights. As it learns multiple weights for every edge, it is far from trivial to combine those weights with the irregular graph structure to perform graph label-relevant feature selection. 

There are two types of attention models: One normalizes the attention weights to sum to one~\cite{bahdanau2014neural}, while the other learns weights between $[0,1]$ without normalization~\cite{xu2015show}. As the counterparts in GNN models, GAT adopts the normalized one~\cite{velivckovic2018graph} while GGNN adopts the unnormalized one~\cite{li2015gated}. Our method belongs to the second category.

\textbf{Graph Neural Network.} GNNs are neural network models that encode graph-structured data into node representations or graph representations. They initialize each node feature representation with its attributes $h_v^{(0)}=X_v$ and then gradually update it by aggregating representations from its neighbors, i.e., $h_v^{(l+1)}\leftarrow q(h_v^{(l)}, \{h_{u}^{(l)}|u:(u,v)\in E\})$ where $q(\cdot)$ denotes a function implemented by NNs~\cite{gilmer2017neural}. Graph representations are often obtained via an aggregation (sum/mean) of node representations.

\textbf{Learning to Explain (L2X).}
L2X~\cite{chen2018learning} studies the feature selection problem in the regular feature space and proposed a mutual information (MI) maximization rule to select a fixed number of features. Specifically, let $I(a;b) \triangleq \sum_{a,b} \mathbb{P}(a,b)\log  \frac{\mathbb{P}(a,b)}{\mathbb{P}(a)\mathbb{P}(b)}$ denote the MI between two random variables $a$ and $b$. Large MI indicates certain high correlation between two random variables. Hence, with input features $X\in\mathbb{R}^F$, L2X is to search a $k$-sized set of indices $S\subseteq\{1,2,...,F\}$, where $k=|S|<F$, such that the features in the subspace indexed by $S$ (denoted by $X_S$) maximizes the mutual information with the labels $Y$, i.e., 
\begin{align} \label{eq:L2X}
    \max_{S\subseteq\{1,2,...,F\}}\; I(X_S;Y), \quad \text{s.t.} \;  |S|\leq k.
\end{align}
Our model is inspired by L2X. However, as graph features and their interpretable counterparts are in an irregular space without a fixed dimension, directly applying L2X may achieve subpar performance in graph learning tasks. We propose to use information constraint instead in Sec.~\ref{sec:obj}.

Later, we will also use the \emph{entropy} defined as $H(a)$ $\triangleq - \sum_{a} \mathbb{P}(a)\log \mathbb{P}(a)$ and the \emph{KL-divergence} defined as $\text{KL}(\mathbb{P}(a)||\mathbb{Q}(a)) \triangleq \sum_{a} \mathbb{P}(a) \log \frac{\mathbb{P}(a)}{\mathbb{Q}(a)} $~\cite{cover1999elements}.

\section{Graph Learning Interpretation via GIB}

In this section, we will first propose the GIB-based objective for interpretable graph learning and point out the issues of post-hoc GNN interpretation methods.  

\subsection{GIB-based Objective for Interpretation}  \label{sec:obj}
Finding label-relevant subgraphs in graph learning tasks has unique challenges. As for the irregularity of graph structures, graph learning models often have to deal with the input graphs of various sizes. The critical subgraph patterns may be also of different sizes and be highly irregular. Consider the example of molecular solubility again, although the functional groups for positive solubility such as -OH, -NH$_2$ are of similar sizes, those for negative solubility range from small groups (e.g., -Cl) to extremely large ones (e.g. -C$_{10}$H$_9$). And, a molecule may contain multiple functional groups scattered in the graph that determine its properties. 
Given these observations, it is not proper to just mimic the cardinality constraint used for a regular dimension space (Eq.~\eqref{eq:L2X}) and select subgraphs of certain sizes potentially with a connectivity constraint as done in~\cite{ying2019gnnexplainer}. Inspired by the graph information bottleneck (GIB) principle~\cite{wu2020graph,yu2020graph}, we propose to use information constraint instead to select label-relevant subgraphs, i.e., solving \vspace{-1mm}
\begin{align} \label{eq:GIB}
    \max_{G_S} I(G_S ; Y), \text{s.t.} \; I(G_S;G) \leq \gamma, G_S\in \mathbb{G}_{sub}(G)
\end{align}
where $\mathbb{G}_{sub}(G)$ denotes the set of the subgraphs of $G$. Note that GIB does not impose any potentially biased constraints such as the size or the connectivity of the selected subgraphs. Instead, GIB uses the information constraint $I(G_S;G) \leq \gamma$ to select $G_S$ that inherits only the most indicative information from $G$ to predict the label $Y$ by maximizing $I(G_S ; Y)$. As thus, $G_S$ provides model interpretation. 

\citet{yu2020graph} also considered using GIB to select subgraphs. However, we adopt a fundamentally different mechanism that we will provide a detailed comparison in Sec.~\ref{sec:comparison}.

\begin{figure}[t]
\vspace{-1mm}
\begin{center}
\centerline{\includegraphics[trim={0.5cm 0cm 2.5cm 0.5cm},clip,width=1.0\columnwidth]{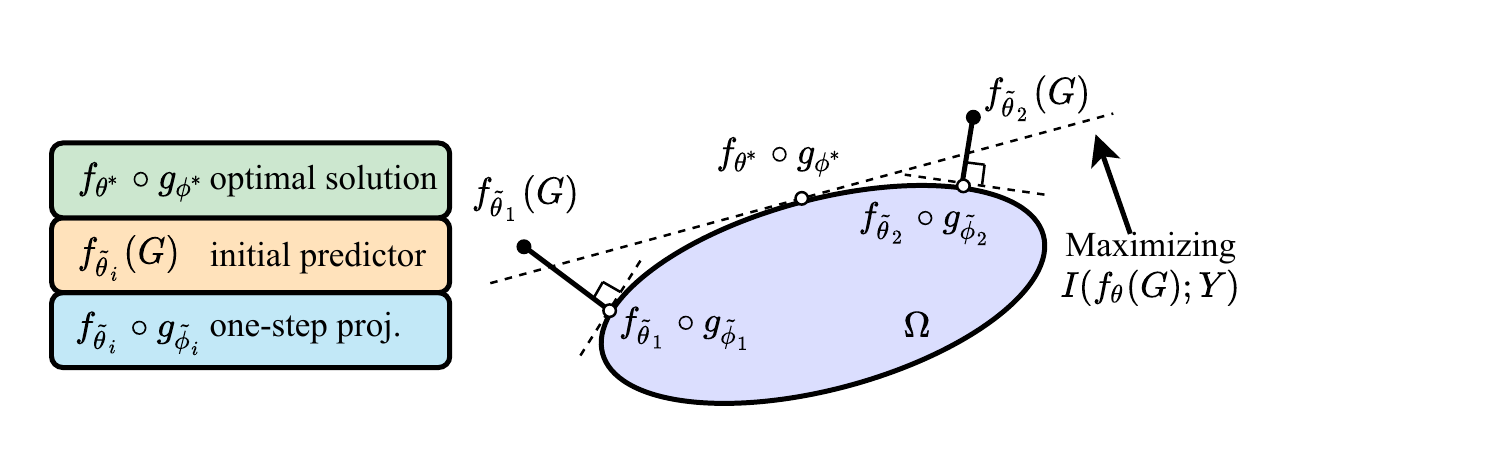}}
\end{center}
\vspace{-13mm}
\caption{Post-hoc methods just perform one-step projection to the information-constrained space, which is always suboptimal and the interpretation performance is sensitive to the pre-trained model.}
\vspace{-1mm}
\label{fig:projection}
\end{figure}

\begin{figure*}[t]
    \vspace{-1mm}
     \centering
     \begin{subfigure}[t]{0.49\linewidth} 
         \centering
         \includegraphics[width=0.49\linewidth]{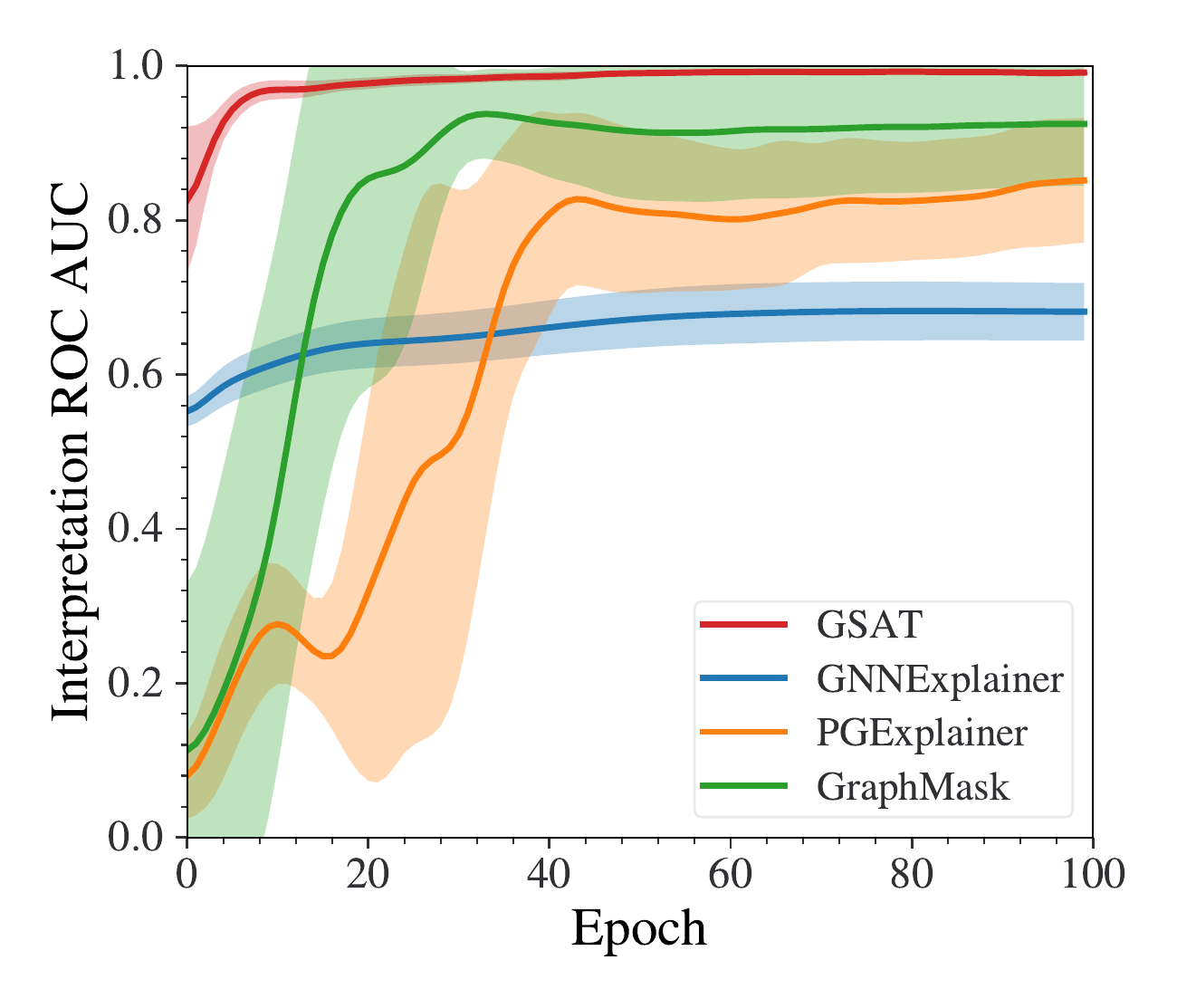}
         \includegraphics[width=0.49\linewidth]{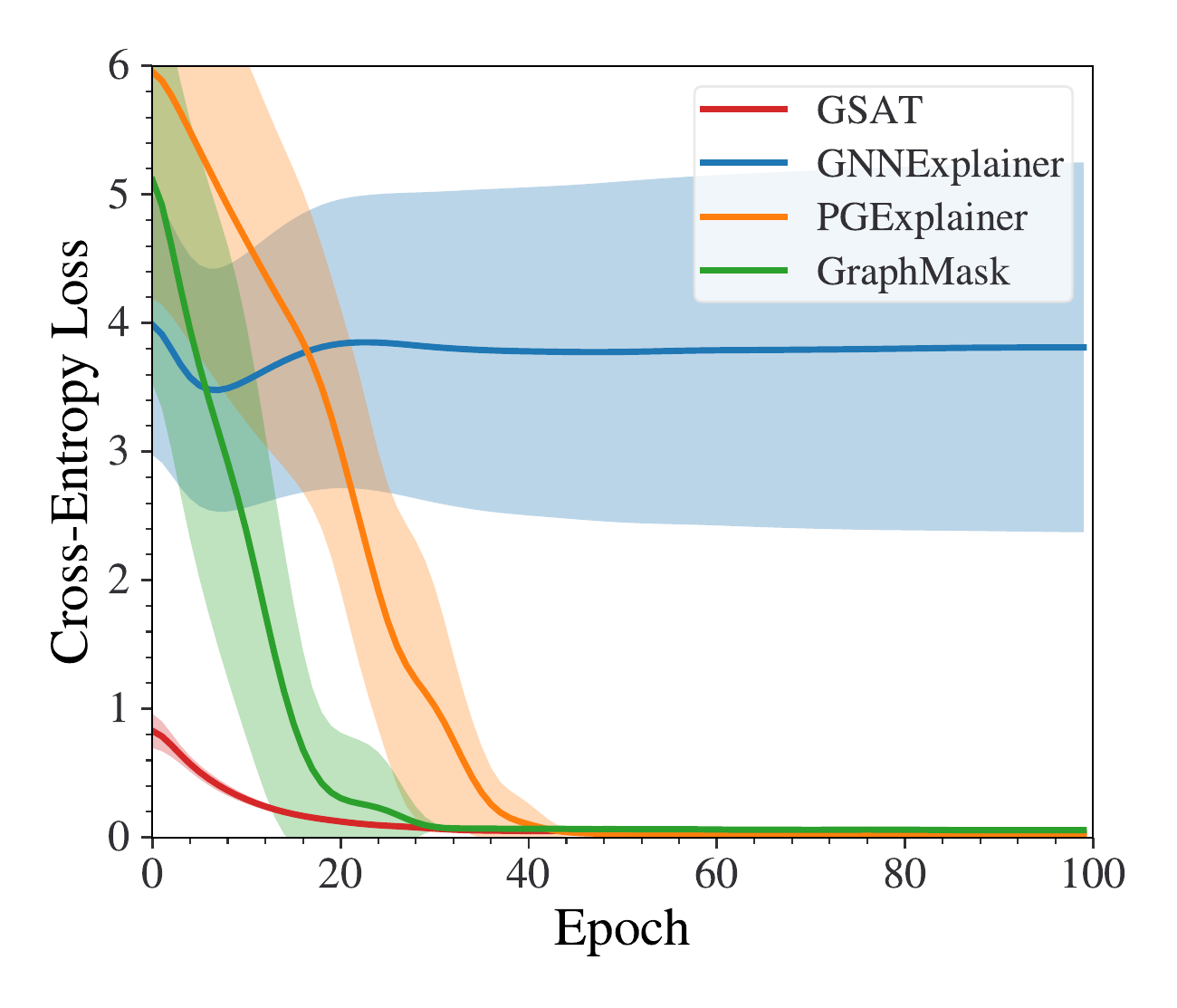}
         \vspace*{-4mm}
         \caption{Ba-2Motifs}
         \label{fig:ba2}
     \end{subfigure}
     \hfill
     \begin{subfigure}[t]{0.49\linewidth}
         \centering
         \includegraphics[width=0.49\linewidth]{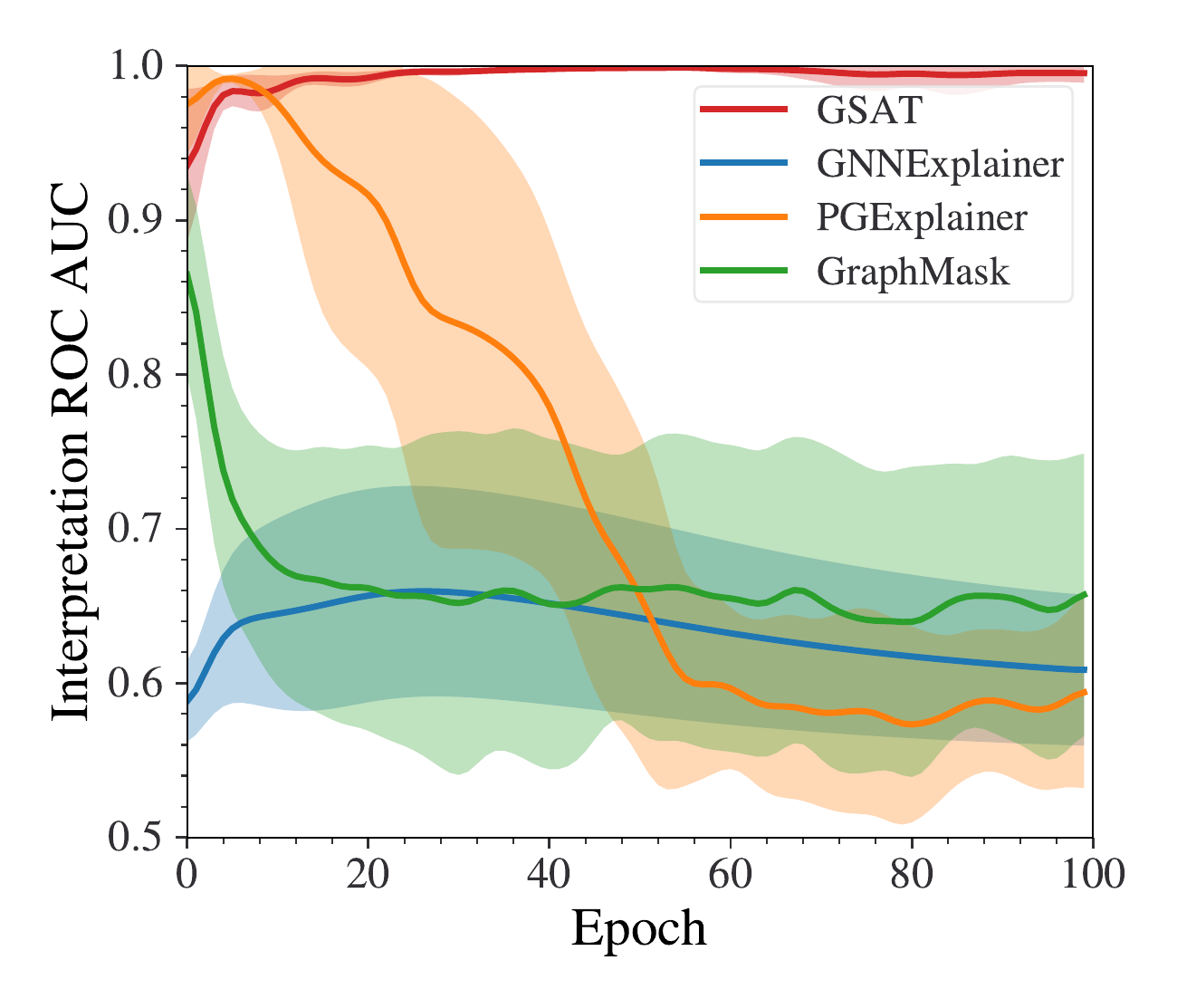}
         \includegraphics[width=0.49\linewidth]{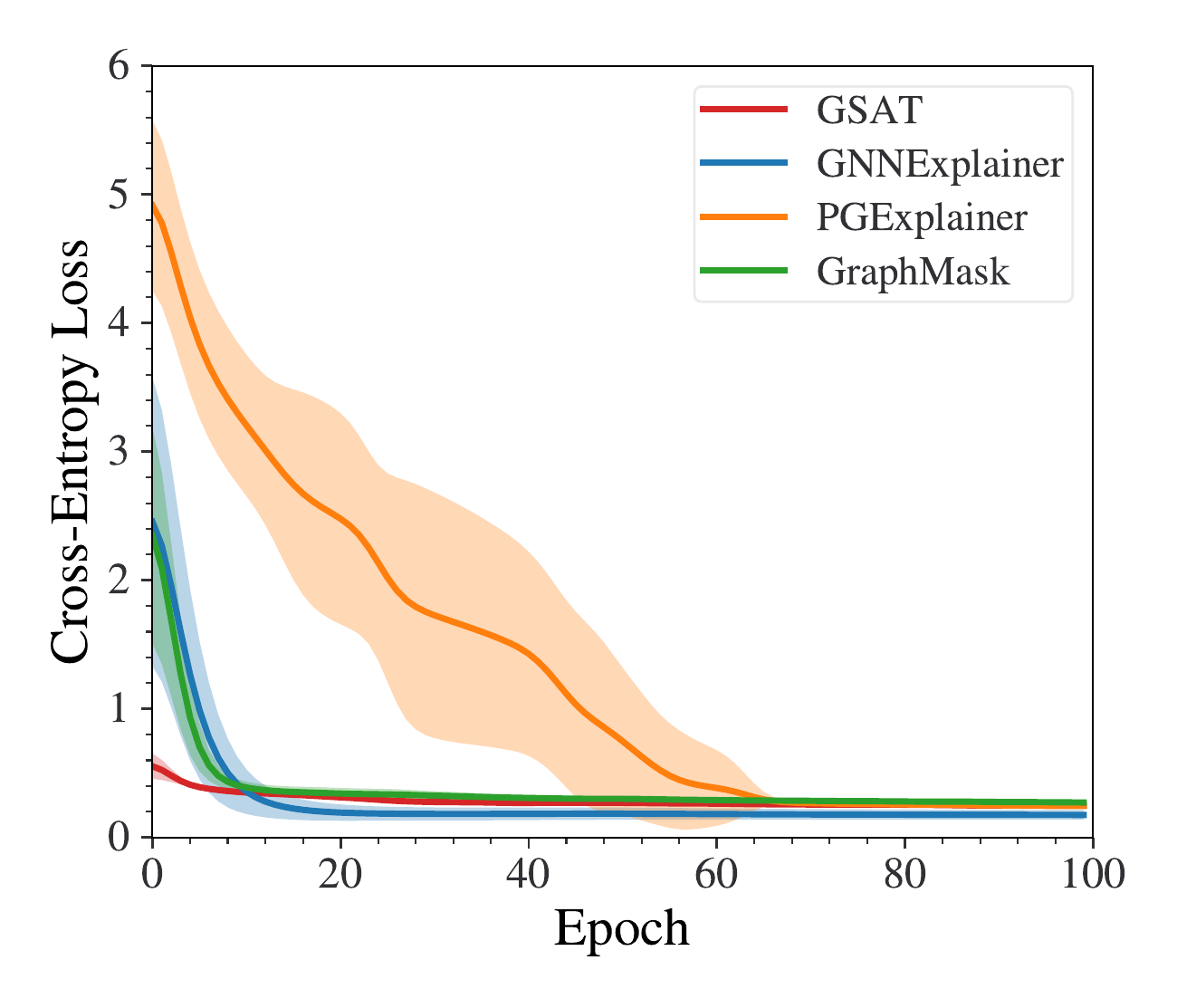}
         \vspace*{-4mm}
         \caption{Mutag}
          \label{fig:mutag}
     \end{subfigure}
     \vspace*{-2mm}
    \caption{Issues of post-hoc interpretation methods. All methods are trained with $10$ random seeds; post-hoc methods are also provided with models pre-trained with different seeds.
    Interpretation performance and the training losses of Eq.~\ref{eq:GIB} for \proj and Eq.~\ref{eq:posthoc} for others are shown. We guarantee that all the pre-trained models are well-trained in their pre-training stage (Acc. $\sim$100\% Ba-2Motif, $\sim$90\% Mutag). }
    \label{fig:post-hoc-fail}
    \vspace*{-2mm}
\end{figure*}

\subsection{Issues of Post-hoc GNN Interpretation Methods} \label{sec:post-hoc}

Almost all previous GNN interpretation methods are post-hoc, such as GNNExplainer~\cite{ying2019gnnexplainer}, PGExplainer~\cite{luo2020parameterized} and GraphMask~\cite{schlichtkrull2021interpreting}. 
Given a pre-trained predictor $f_{\theta}(\cdot): \mathcal{G}\rightarrow \mathcal{Y}$, they try to find out the subgraph $G_S$ that impacts the model predictions the most, while keeping the pre-trained model unchanged. This procedure essentially first maximizes the MI between $f_{\theta}(G)$ and $Y$ and obtains a model parameter\vspace{-1mm}
\begin{align} \label{eq:pretrain}
    \tilde{\theta} \triangleq \arg\max_{\theta} I(f_{\theta}(G);Y),
\end{align}
and then optimizes a subgraph extractor $g_{\phi}$ via
\begin{align} \label{eq:posthoc}\vspace{-1mm}
     \tilde{\phi} \triangleq\arg\max_{\phi} I(f_{\tilde{\theta}}(G_S);Y), \, \text{s.t.}\; G_S=g_{\phi}(G)\in \Omega.
\end{align}
where $\Omega$ implies a subset of the subgraphs $\mathbb{G}_{sub}(G)$ that satisfy some constraints, e.g., the cardinality constraint adopted by GNNExplainer and PGExplainer. Let us temporarily ignore the difference between different constraints and just focus on the optimization objective. The post-hoc objective Eq.~\eqref{eq:posthoc} and GIB (Eq.~\eqref{eq:GIB}) share some similar spirits. However, the post-hoc methods may not give or even approximate the optimal solution to Eq.~\eqref{eq:GIB} because $f_{\theta}\circ g_{\phi}$ is not jointly trained. From the optimization perspective, post-hoc methods just perform \emph{one-single step projection} (see Fig.~\ref{fig:projection}) from the model $f_{\tilde{\theta}}$ in an unconstrained space to $f_{\tilde{\theta}}\circ g_{\tilde{\phi}}$ in the information-constrained space $\Omega$ where the projection rule follows that the induced MI decrease $I(f_{\tilde{\theta}}(G);Y) - I(f_{\tilde{\theta}}(g_{\tilde{\phi}}(G));Y)$ gets minimized.

In practice, such a suboptimal behavior will yield two undesired consequences. First, $f_{\tilde{\theta}}$ may not fully extract the information from $G_S=g_{\phi}(G)$ to predict $Y$ during the optimization of Eq.~\eqref{eq:posthoc} because $f_{\tilde{\theta}}$ is originally trained to make $I(f_{\tilde{\theta}}(G);Y)$ approximate $I(G,Y)$ while $(G_S, Y) = (g_{\phi}(G), Y)$ follows a distribution different from $(G,Y)$. Therefore, $I(f_{\tilde{\theta}}(G_S);Y)$ may not well approximate $I(G_S;Y)$, 
and thus may mislead the optimization of $g_{\phi}$ and disable $g_{\phi}$ to select $G_S$ that indeed indicates $Y$. GNNExplainer suffers from this issue over Ba-2Motif as shown in Fig.~\ref{fig:post-hoc-fail}: The training loss, $-I(f_{\tilde{\theta}}(G_S);Y)$ keeps high and the interpretation performance is subpar. It is possible to further decrease the training loss via a more aggressive optimization of $g_\phi$. However, the models may risk overfitting the data, which yields the second issue. 


An aggressive optimization of $g_{\phi}$ may give a large \emph{empirical} MI $\hat{I}\left(f_{\tilde{\theta}}(g_{\phi}(G)); Y\right)$ (or a small training loss equivalently) by selecting features that help to distinguish labels for training but are essentially irrelevant to the labels or spuriously correlated with the labels in the population level. 
Previous works have shown that label-irrelevant features are known to be discriminative enough to even identify each graph in the training dataset let alone the labels~\cite{suresh2021adversarial}. Empirically, we indeed observe such overfitting problems of all post-hoc methods over Mutag as shown in Fig.~\ref{fig:post-hoc-fail}, especially PGExplainer and GraphMask. In the first $5$ to $10$ epochs, these two models succeed in selecting good explanations while having a large training loss. Further training successfully decreases the loss (after $10$ epochs) but degenerates the interpretation performance substantially. 
This might also be the reason why in the original literatures of these post-hoc methods, training over only a small number of epochs is suggested. However, 
in practical tasks, it is hard to have the ground truth interpretation labels to verify the results and decide a trusty stopping criterion. 

Another observation of Fig.~\ref{fig:post-hoc-fail} also matches our expectation: From the optimization perspective, post-hoc methods suffer from an initialization issue. Their interpretability can be highly sensitive to the pre-trained model $f_{\tilde{\theta}}$, as empirically demonstrated by the large variances in Fig.~\ref{fig:post-hoc-fail}. Only if the pre-trained $f_{\tilde{\theta}}$ approximates the optimal $f_{\theta^*}$, the performance can be roughly guaranteed. So, a joint training of $f_{\theta}\circ g_{\phi}$ according to the GIB principle Eq.~\eqref{eq:GIB} is typically needed. 

\section{Stochastic Attention Mechanism for GIB} \label{sec:method}

In this section, we will first give a tractable variational bound of the GIB objective (Eq.~\eqref{eq:GIB}), and then introduce our model \proj with the stochastic attention mechanism. We will further discuss how the stochastic attention mechanism improves both model interpretation and generalization. 

\subsection{A Tractable Objective for GIB}

\proj is to learn an extractor $g_{\phi}$ with parameter $\phi$ to extract $G_S\in \mathbb{G}_{\text{sub}}(G)$. $g_{\phi}$ blocks the label-irrelevant information in the data $G$ via injected stochasticity while allowing the label-relevant information kept in $G_S$ to make predictions. In \proj, $g_{\phi}(G)$ essentially gives a distribution over $\mathbb{G}_{\text{sub}}(G)$. We also denote this distribution as $\mathbb{P}_{\phi}(G_S|G)$. Later, $g_{\phi}(G)$ and $\mathbb{P}_{\phi}(G_S|G)$ are used interchangeably. 

Putting the constraint into the objective (Eq.\eqref{eq:GIB}), we obtain the optimization of $g_{\phi}$ via GIB, i.e., for some $\beta>0$,
\begin{align} \label{eq:GIB2}
    \min_{\phi} -I(G_S ; Y) + \beta I(G_S;G),\, \text{s.t.}\,\, G_S \sim g_{\phi}(G). 
\end{align}
Next, we follow~\citet{alemi2016deep, poole2019variational,wu2020graph} to derive a tractable variational upper bound of the two terms in Eq.~\eqref{eq:GIB2}. Detailed derivation is given in Appendix~\ref{appx:deriving}. For the term $I\left(G_S ; Y\right)$, we introduce a parameterized variational approximation $\mathbb{P}_{\theta}(Y|G_S)$ for $\mathbb{P}(Y|G_S)$. We obtain a lower bound:
\begin{align} \label{eq:predictor}
    I\left(G_S ; Y\right) \geq \mathbb{E}_{G_S, Y} \left[ \log {\mathbb{P}_{\theta}(Y|G_S)} \right] + H(Y).
\end{align}
Note that $\mathbb{P}_{\theta}(Y|G_S)$ essentially works as the predictor  $f_{\theta}:\mathcal{G}\rightarrow \mathcal{Y}$ with parameter $\theta$ in our model. 
For the term $I(G_S;G)$, we introduce a variational approximation $\mathbb{Q}(G_S)$ for the marginal distribution $\mathbb{P}(G_S) = \sum_{G}\mathbb{P}_{\phi}(G_S|G)\mathbb{P}_{\mathcal{G}}(G)$. And, we obtain an upper bound:
\begin{align} \label{eq:extractor}
    I\left(G_s ; G\right) \leq \mathbb{E}_{G}\left[\text{KL}(\mathbb{P}_{\phi}(G_S|G)||\mathbb{Q}(G_S)) \right]
\end{align}
Plugging in the above two inequalities, we obtain a variational upper bound of Eq.~\eqref{eq:GIB2} as the objective of \proj: 
\begin{align} \nonumber  
    &\min_{\theta,\phi}\,-\mathbb{E}\left[ \log {\mathbb{P}_{\theta}(Y|G_S)} \right]  + \beta \mathbb{E}\left[\text{KL}(\mathbb{P}_{\phi}(G_S|G)||\mathbb{Q}(G_S)) \right], 
    \\ &\quad\quad \text{s.t.}\quad  G_S \sim \mathbb{P}_{\phi}(G_S|G). \label{eq:proj}   
\end{align}
Next, we specify $\mathbb{P}_{\theta}$ (aka $f_{\theta}$), $\mathbb{P}_{\phi}$ (aka $g_{\phi}$) and $\mathbb{Q}$ in \proj.

\subsection{\proj and Stochastic Attention Mechanism}
For clarity, we introduced the predictor $f_{\theta}$ and the extractor $g_{\phi}$ separately. Actually, \proj is a unified model as $f_{\theta},\,g_{\phi}$ share the same GNN encoder except their last layers.

\textbf{Stochastic Attention via $\mathbb{P}_{\phi}$.} The extractor $g_{\phi}$ first encodes the input graph $G$ via the GNN into a set of node representations $\{h_v|v\in V\}$. For each edge $(u,v)\in E$, $g_{\phi}$ contains an MLP layer plus sigmoid that maps the concatenation $(h_u,h_v)$ into $p_{uv}\in [0,1]$. Then, for each forward pass of the training, we sample stochastic attention from Bernoulli distributions $\alpha_{uv}\sim \text{Bern}(p_{uv})$. To make sure the gradient w.r.t. $p_{uv}$ is computable, we apply the gumbel-softmax reparameterization trick~\cite{jang2016categorical}. The extracted graph $G_S$ will have an attention-selected subgraph as $A_S = \alpha \odot A$. Here $\alpha$ is the matrix with entries $\alpha_{uv}$ for $(u,v)\in E$ or zeros for the non-edge entries. $A$ is the adjacency matrix of $G$ and $\odot$ is entry-wise product. The distribution of $G_S$ given $G$ through the above procedure characterizes $\mathbb{P}_{\phi}(G_S|G)$, so $\mathbb{P}_{\phi}(G_S|G) = \prod_{u,v\in E}\mathbb{P}(\alpha_{uv}|p_{uv})$, where $p_{uv}$ is a function of $G$. This essentially makes the attention $\alpha_{uv}$ to be conditionally independent across different edges given the input graph $G$.

\textbf{Prediction via $\mathbb{P}_{\theta}$.} The predictor $f_{\theta}$ adopts the same GNN to encode the extracted graph $G_S$ to a graph representation, and finally passes such representation through an MLP layer plus softmax to model the distribution of $Y$. This procedure gives the variational distribution $\mathbb{P}_{\theta}(Y|G_S)$. 

\textbf{Marginal Distribution Control via $\mathbb{Q}$.} 
The bound Eq.\eqref{eq:extractor} is always true for any $\mathbb{Q}(G_S)$. We define $\mathbb{Q}(G_S)$ as follows. For every graph $G\sim \mathbb{P}_{\mathcal{G}}$ and every two directed node pair $(u,v)$ in $G$, we sample $\alpha_{uv}'\sim \text{Bern}(r)$ where $r\in[0,1]$ is a hyperparameter. 
We remove all edges in $G$ and add all edges $(u,v)$ if $\alpha_{uv}'=1$.
Suppose the obtained graph is $G_S$. This procedure defines the distribution $\mathbb{Q}(G_S) = \sum_{G}\mathbb{P}(\alpha'|G)\mathbb{P}_{\mathcal{G}}(G)$. As $\alpha'$ is independent from the graph $G$ given its size $n$, $\mathbb{Q}(G_S) = \sum_{n}\mathbb{P}(\alpha'|n)\mathbb{P}_{\mathcal{G}}(G=n) = \mathbb{P}(n)\prod_{u,v=1}^n \mathbb{P}(\alpha_{uv}')$. The probability of an $n$-sized  graph $\mathbb{P}(n)$ is a constant and thus will not affect the model. Note that our choice of $\mathbb{Q}(G_S)$ shares the similar spirit of using standard Gaussian as the latent distribution with variational auto-encoders~\cite{kingma2013auto}. 

Using the above $\mathbb{P}_{\theta}$, the first term in Eq.\eqref{eq:proj} reduces to a standard cross entropy loss. Using $\mathbb{P}_{\phi}$ and $\mathbb{Q}$, the KL-divergence term becomes, for every $G\sim \mathbb{P}_{\mathcal{G}}$, $n$ as the size of $G$,
\begin{align}\label{eq:reg}
    &\text{KL}(\mathbb{P}_{\phi}(G_S|G)||\mathbb{Q}(G_S))  = \\
    &\sum_{(u,v)\in E} p_{uv} \log \frac{p_{uv}}{r} + \left(1-p_{uv}\right) \log \frac{1-p_{uv}}{1-r} + c(n,r). \nonumber
\end{align}
where $c(n,r)$ is a constant without any trainable parameters. 

\subsection{The Interpretation Mechanism of \proj}

The interpretability of \proj essentially comes from the information control: \proj decreases the information from the input graphs by injecting stochasticity via attention into $G_S$. In the training, the regularization term Eq.\eqref{eq:reg} would try to assign large stochasticity for all edges, yet driven by the classification loss $\min -I(G_S;Y)$ (equivalent to cross-entropy loss), \proj can learn to reduce such stochasticity of the attention on the task-relevant subgraphs. So, it is not the entire $G_S$ but the part of $G_S$ with the stochasticity-reduced attention, aka $p_{uv} \rightarrow 1$, that provide model interpretation. Therefore, when \proj provides interpretation, in practice, one can rank all edges according to $p_{uv}$ and use those top ranked ones (given a certain budget if needed) as the detected subgraph for interpretation. 
The contribution of injecting stochasticity to the performance is so significant as shown in experiments (Table~\ref{table:ab-beta-noise-gin}), so is the contribution of our regularization term (Eq.~\eqref{eq:reg}) when we compare it with the sparsity-driven $\ell_1$-norm (Fig.~\ref{fig:reg-compare}).

\proj is substantially different from previous methods, as we do not use any sparsity constraints such as $\ell_1$-norm~\cite{ying2019gnnexplainer,luo2020parameterized}, $\ell_0$-norm~\cite{schlichtkrull2021interpreting} or $\ell_2$-regression to $\{0,1\}$~\cite{yu2020graph} to select size-constrained (or connectivity-constrained) subgraphs. We actually observe that setting $r$ away from 0 in the marginal regularization (Eq.~\eqref{eq:reg}), i.e., pushing $G_S$ away from being sparse often provides more robust interpretation. 
This matches our intuition that GIB by definition does not make any assumptions on the selected subgraphs but just constrains the information from the original graphs. Our experiments show that \proj outperform baselines significantly without leveraging those assumptions in the optimization even if the label-relevant subgraphs satisfy these assumptions. If the label-relevant subgraphs are indeed disconnected or vary in sizes, the improvement of \proj is expected to be even more.

\subsection{Further Comparison on Interpretation Mechanism} \label{sec:comparison}

PGExplainer and GraphMask also have stochasticity in their models~\cite{luo2020parameterized,schlichtkrull2021interpreting}. However, their main goal is to enable a gradient-based search over a discrete subgraph-selection space rather than control the information as \proj does. Hence, they did not in principle derive the information regularization as ours (Eq.~\eqref{eq:reg}) but adopt sparsity constraints to extract a small subgraph $G_S$ directly used for interpretation. 

IB-subgraph~\cite{yu2020graph} considers using GIB as the objective but does not inject any stochasticity to generate $G_S$, so its selected subgraph $G_S$ is a deterministic function of $G$. Specifically, IB-subgraph samples batches of graphs $G$ to estimate $I(G_S;G)$ and optimize a deterministic function $G_S=g_{\phi}(G)$ to minimize such MI estimation. 
In this case $I(G_S;G) (= H(G_S) - H(G_S|G))$ reduces to the entropy $H(G_S)$, 
which tends to give a small-sized $G_S$, because the space of small graphs is small and has a lower upper bound of the entropy.
By contrast, $G_{S} \sim g_{\phi}(G)$ is random in \proj, and \proj implements GIB mainly by increasing $H(G_S|G)$ via injecting stochasticity.

\subsection{Guaranteed Spurious Correlation Removal}
\label{sec:better-generalization}
\proj can remove spurious correlations in the training data and has guaranteed interpretability. We may prove that if there exists a correspondence between a subgraph pattern $G_S^*$ and the label $Y$, 
the pattern $G_S^*$ is the optimal solution of the GIB objective (Eq.~\eqref{eq:GIB}).

\begin{figure}[t]
\begin{center}
\centerline{\includegraphics[trim={0.6cm 0.2cm 0.1cm 0cm},clip,width=0.8\columnwidth]{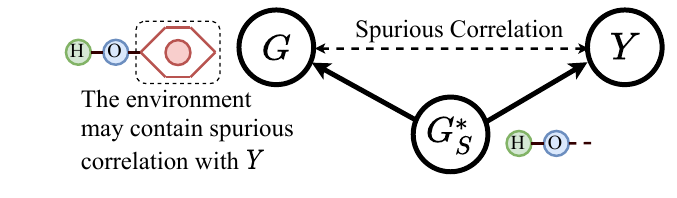}}
\end{center}
\vspace{-10mm}
\caption{$G_S^*$ determines $Y$. However, the environment features in $G\backslash G_S^*$ may contain spurious (backdoor) correlation with $Y$.}
\label{fig:sp_correlation}
\vspace{-1mm}
\end{figure}

\begin{theorem}
\label{thm:iboptim}
Suppose each $G$ contains a subgraph $G_S^*$ such that $Y$ is determined by $G_S^*$ in the sense that $Y=f(G_S^*)+\epsilon$ for some deterministic invertible function $f$ with randomness $\epsilon$ that is independent from $G$. Then, for any $\beta\in [0,1]$,  $G_S=G_S^*$ maximizes the GIB $I\left(G_S ; Y\right) -  \beta I \left(G_S;G\right) $, where $G_S\in \mathbb{G}_{\text{sub}}(G)$.
\end{theorem}
\begin{proof}
Consider the following derivation:
\begin{align*}
           & I(G_S; Y)  -  \beta I(G_S; G) \\
        =  & I(Y;G,G_S) - I(G;Y|G_S) - \beta I(G_S; G) \\
        =  & I(Y;G,G_S) - (1-\beta) I(G;Y|G_S) -  \beta I(G;G_S,Y) \\
        =  & I(Y;G) - (1-\beta) I(G;Y|G_S) - \beta I(G;G_S,Y) \\
        = & (1-\beta) I(Y;G) - (1-\beta) I(G;Y|G_S) - \beta I(G;G_S|Y),
\end{align*}
where the third equality is because $G_S\in \mathbb{G}_{sub}(G)$, then $(G_S, G)$ holds no more information than $G$. 

If $\beta\in[0,1]$, $G_S$ that maximizes $I(G_S, Y)  -  \beta I(G_S; G)$ can also minimize $(1-\beta) I(G;Y|G_S) + \beta I(G;G_S|Y)$. As $I(G;Y|G_S)\geq 0$, $I(G;G_S|Y)\geq 0$, the lower bound of $(1-\beta) I(G;Y|G_S) + \beta I(G;G_S|Y)$ is 0. 

$G_S^*$ is the subgraph that makes $(1-\beta) I(G;Y|G_S^*) + \beta I(G;G_S^*|Y)= 0$. This is because (a) $Y=f(G_S^*)+\epsilon$ where $\epsilon$ is independent of $G$ so $I(G;Y|G_S^*) = 0$ and (b) $G_S^* = f^{-1}(Y-\epsilon)$ where $\epsilon$ is independent of $G$ so $I(G;G_S^*|Y) = 0$. Therefore, $G_S=G_S^*$ maximizes GIB $I\left(G_S ; Y\right) -  \beta I \left(G_S;G\right)$, where $G_S\in \mathbb{G}_{\text{sub}}(G)$.
\end{proof}

Although $G_S^*$ determines $Y$, in the training dataset the data $G$ and $Y$ may have some spurious correlation caused by the environment~\cite{pearl2016causal,arjovsky2019invariant,chang2020invariant,krueger21a}. That is, $G\backslash G_S^*$ may have some correlation with the label, but this correlation is spurious and is not the true reason that determines its label  (illustrated in Fig.~\ref{fig:sp_correlation}). A model trained over $G$ to predict $Y$ via just MI maximization may capture such spurious correlation. If such correlation is changed during the test phase, the model suffers from performance decay. 

However, Theorem~\ref{thm:iboptim} indicates that \proj by optimizing the GIB objective has the capability to address the above issue by only extracting $G_S^*$, which removes the spurious correlation and also provides guaranteed interpretability.



\vspace{-1mm}
\subsection{Fine-tuning and Interpreting a Pre-trained Model} \label{sec:finetune}

\proj can also fine-tune and interpret a pre-trained GNN. Given a GNN $f_{\tilde{\theta}}$ pre-trained by $\max _{\theta} I(f_{\theta}(G) ; Y)$, \proj can fine-tune it via $\max _{\theta,\phi} I(f_{\theta}(G_S) ; Y) - \beta I(G_S;G)$, $G_S\sim g_{\phi}(G)$ by initializing the GNN used in $g_{\phi}$ and $f_{\theta}$ as the one in the pre-trained model $f_{\tilde{\theta}}$.

We observe that this framework almost never hurts the original prediction performance (and sometimes even boosts it). Moreover, this framework often achieves better interpretation results compared with training the GNN from scratch.

\vspace{-1mm}
\section{Other Related Works}

Besides the models \cite{ying2019gnnexplainer,luo2020parameterized,schlichtkrull2021interpreting,yu2020graph} that we have compared with in detail in Sec.~\ref{sec:post-hoc} and Sec.~\ref{sec:comparison}, we review some other interpretation methods here.

Most previous works on GNN interpretation are post-hoc~\cite{ribeiro2016model}. Some works strongly rely on the connectivity assumption and only search over the space of connected subgraphs for interpretation. They adopt either reinforcement learning~\cite{yuan2020xgnn} or Monte Carlo tree search~\cite{yuan2021explainability}. Other methods including PGM-Explainer~\cite{vu2020pgm} leveraging graphical models, Gem~\cite{lin2021generative} checking Granger causality and Graphlime~\cite{huang2020graphlime} using HSIC Lasso are only applied to node-level task interpretation. Some works check the gradients w.r.t. the input features to find important features~\cite{pope2019explainability, baldassarre2019explainability}. 


Much fewer works have considered intrinsic interpretation. Recently, \citet{anonymous2022discovering} has proposed DIR to make the model avoid overfitting spurious correlations and only capture invariant rationales to provide interpretability. However, DIR needs to iteratively break graphs into subgraphs and assemble subgraphs into graphs during the model training, which is far more complicated than \proj. 
\vspace{-1mm}
\section{Experiments} \label{experiments}

\begin{table*}[t]
\vspace{-2mm}
\caption{Interpretation Performance (AUC). The \underline{underlined} results highlight the best baselines. The \textbf{bold} font and \textbf{bold}$^{\dagger}$ font highlight when \proj outperform the means of the best baselines based on the mean of \proj and the mean-2*std of \proj, respectively. }
\vspace{-0.1cm}
\begin{center}
\begin{small}
\begin{sc}
\begin{tabular}{lcccccc}
\toprule
  & \multirow{2}{*}{Ba-2motifs} & \multirow{2}{*}{Mutag} & \multirow{2}{*}{MNIST-75sp} & \multicolumn{3}{c}{Spurious-motif}                                        \\
  &                    &                    &                    & $b=0.5$                  & $b=0.7$                  & $b=0.9$                  \\
\midrule
GNNExplainer & $67.35\pm3.29$ & $61.98\pm5.45$ & $59.01\pm2.04$ & $62.62\pm1.35$ & $62.25\pm3.61$ & $58.86\pm1.93$ \\
PGExplainer  & $84.59\pm9.09$ & $60.91\pm17.10$ & $69.34\pm4.32$ & $69.54\pm5.64$ & $72.33\pm9.18$ & $\underline{72.34}\pm2.91$ \\
GraphMask    & $\underline{92.54}\pm8.07$ & $62.23\pm9.01$ & $\underline{73.10}\pm6.41$ & $72.06\pm5.58$ & $73.06\pm4.91$ & $66.68\pm6.96$ \\
IB-Subgraph          & $86.06\pm28.37$ & $\underline{91.04}\pm6.59$ & $51.20\pm5.12$ & $57.29\pm14.35$ & $62.89\pm15.59$ & $47.29\pm13.39$ \\
DIR          & $82.78\pm10.97$ & $64.44\pm28.81$ & $32.35\pm9.39$ & $\underline{78.15}\pm1.32$ & $\underline{77.68}\pm1.22$ & $49.08\pm3.66$ \\
\midrule
GIN+\proj    & $\mathbf{98.74}^\dagger\pm0.55$ & $\mathbf{99.60}^\dagger\pm0.51$ & $\mathbf{83.36}^\dagger\pm1.02$ & $\mathbf{78.45}\pm3.12$ & $74.07\pm5.28$ & $71.97\pm4.41$ \\
GIN+$\text{\proj}^*$ & $\mathbf{97.43}^\dagger\pm1.77$ & $\mathbf{97.75}^\dagger\pm0.92$ & $\mathbf{83.70}^\dagger\pm1.46$ & $\mathbf{85.55}^\dagger\pm2.57$ & $\mathbf{85.56}^\dagger\pm1.93$ & $\mathbf{83.59}^\dagger\pm2.56$ \\
\midrule
PNA+\proj    & $\mathbf{93.77}\pm3.90$ & $\mathbf{99.07}^\dagger\pm0.50$ & $\mathbf{84.68}^\dagger\pm1.06$ & $\mathbf{83.34}^\dagger\pm2.17$ & $\mathbf{86.94}^\dagger\pm4.05$ & $\mathbf{88.66}^\dagger\pm2.44$ \\
PNA+$\text{\proj}^*$ & $89.04\pm4.92$ & $\mathbf{96.22}^\dagger\pm2.08$ & $\mathbf{88.54}^\dagger\pm0.72$ & $\mathbf{90.55}^\dagger\pm1.48$ & $\mathbf{89.79}^\dagger\pm1.91$ & $\mathbf{89.54}^\dagger\pm1.78$ \\
\bottomrule
\label{table:Interpretation}
\end{tabular}
\end{sc}
\end{small}
\end{center}
\vspace{-9mm}
\end{table*}

\begin{table*}[t]
\caption{Prediction Performance (Acc.). The \textbf{bold} font highlights the inherently interpretable methods that significantly outperform the corresponding backbone model, GIN or PNA, when the mean-1*std of a method $>$ the mean of its corresponding backbone model.}
\vspace{-0.1cm}
\begin{center}
\begin{small}
\begin{sc}
\begin{tabular}{lcccccc}
\toprule
  & \multirow{2}{*}{MolHiv (AUC)} & \multirow{2}{*}{Graph-SST2} & \multirow{2}{*}{MNIST-75sp} & \multicolumn{3}{c}{Spurious-motif}                                        \\
  &                    &                    &                    & $b=0.5$                  & $b=0.7$                  & $b=0.9$                  \\
\midrule
GIN & $76.69\pm1.25$ & $82.73\pm0.77$ & $95.74\pm0.36$ & $39.87\pm1.30$ & $39.04\pm1.62$ & $38.57\pm2.31$ \\
IB-subgraph & ${76.43}\pm2.65$ & $82.99\pm0.67$ & $93.10\pm1.32$ & $\mathbf{54.36}\pm7.09$ & $\mathbf{48.51}\pm5.76$ & $\mathbf{46.19}\pm5.63$ \\
DIR & $76.34\pm1.01$ & $82.32\pm0.85$ & $88.51\pm2.57$ & $\mathbf{45.49}\pm3.81$ & $41.13\pm2.62$ & $37.61\pm2.02$ \\
GIN+\proj & $76.47\pm1.53$ & $82.95\pm0.58$ & $\mathbf{96.24}\pm0.17$ & $\mathbf{52.74}\pm4.08$ & $\mathbf{49.12}\pm3.29$ & $\mathbf{44.22}\pm5.57 $\\
GIN+$\text{\proj}^*$ & $76.16\pm1.39$ & $82.57\pm0.71$ & $\mathbf{96.21}\pm0.14$ & $\mathbf{46.62}\pm2.95$ & $41.26\pm3.01$ & $39.74\pm2.20 $\\
\midrule
PNA (no scalars) & $78.91\pm1.04$ & $79.87\pm1.02$ & $87.20\pm5.61$ & $68.15\pm2.39$ & $66.35\pm3.34$ & $61.40\pm3.56 $ \\
PNA+\proj & $\mathbf{80.24}\pm0.73$ & $\mathbf{80.92}\pm0.66$ & $\mathbf{93.96}\pm0.92$ & $68.74\pm2.24$ & $64.38\pm3.20$ & $57.01\pm2.95 $\\
PNA+$\text{\proj}^*$ & $\mathbf{80.67}\pm0.95$ & $\mathbf{82.81}\pm0.56$ & $\mathbf{92.38}\pm1.44$ & $\mathbf{69.72}\pm1.93$ & $\mathbf{67.31}\pm1.86$ & $61.49\pm3.46 $\\ 
\bottomrule
\label{table:Generalization}
\end{tabular}
\end{sc}
\end{small}
\end{center}
\vskip -6mm
\end{table*}

We evaluate our method for both interpretability and prediction performance. We will compare our method with both state-of-the-art (SOTA) post-hoc interpretation methods and inherently interpretable models. 
We will also compare with several invariant learning methods to demonstrate the ability of \proj to remove spurious correlations. We briefly introduce datasets, baselines and experiment settings here, and more details can be found in Appendix~\ref{appx:setting}.
\vspace{-2mm}
\subsection{Datasets}

\textbf{Mutag}~\cite{debnath1991structure} is a molecular property prediction dataset. Following~\cite{luo2020parameterized}, -NO$_2$ and -NH$_2$ in mutagen graphs are labeled as ground-truth explanations.

\textbf{BA-2Motifs}~\cite{luo2020parameterized} is a synthetic dataset with binary graph labels. House motifs and cycle motifs give class labels and thus are regarded as ground-truth explanations for the two classes respectively.

\textbf{Spurious-Motif}~\cite{anonymous2022discovering} is a synthetic dataset with three graph classes. Each class contains a particular motif that can be regarded as the ground-truth explanation. Some spurious correlation between the rest graph components (other than the motifs) and the labels also exists in the training data. 
The degree of such correlation is controlled by $b$, and we include datasets with $b=0.5$, $0.7$ and $0.9$.

\textbf{MNIST-75sp} ~\cite{knyazev2019understanding} is an image classification dataset, where each image in MNIST is converted to a superpixel graph. Nodes with nonzero pixel values provide ground-truth explanations. Note that the subgraphs that provide explanations are of different sizes in this dataset.
    
\textbf{Graph-SST2} ~\cite{socher2013recursive, yuan2020explainability} is a sentiment analysis dataset, where each text sequence in SST2 is converted to a graph. Following the splits in~\cite{anonymous2022discovering}, this dataset contains degree shifts and no ground-truth explanation labels. So, we only evaluate prediction performance and provide interpretation visualizations.

\textbf{OGBG-Molhiv} ~\cite{wu2018moleculenet, hu2020ogb} is a molecular property prediction datasets. We also evaluate \proj on molbace, molbbbp, molclintox, moltox21 and molsider datasets from OGBG. As there are no ground truth explanation labels for these datasets, we only evaluate the prediction performance of \proj.  

\subsection{Baselines and Setup}
\textbf{Interpretability Baselines.} We compare interpretability with post-hoc methods GNNExplainer ~\cite{ying2019gnnexplainer}, PGExplainer ~\cite{luo2020parameterized}, GraphMask ~\cite{schlichtkrull2021interpreting}, and inherently interpretable models DIR \cite{anonymous2022discovering} and IB-subgraph ~\cite{yu2020graph}.

\textbf{Prediction Baselines.} We compare prediction performance with the  backbone models GIN ~\cite{xu2018powerful} and PNA ~\cite{corso2020principal}, and inherently interpretable models DIR \cite{anonymous2022discovering} and IB-subgraph ~\cite{yu2020graph}.

\textbf{Invariant Learning Baselines.} We compare the ability to remove spurious correlations with invariant learning methods IRM~\cite{arjovsky2019invariant}, V-REx~\cite{krueger21a} and DIR \cite{anonymous2022discovering}. Baseline results yielded by empirical risk minimization (ERM) are also included.

\textbf{Metrics.} For interpretation evaluation, we report explanation ROC AUC following ~\cite{ying2019gnnexplainer, luo2020parameterized}. For prediction performance, we report classification ROC AUC for all OGBG datasets and report accuracy for all other datasets. All the results are averaged over 10 times tests with different random seeds. For the post-hoc methods, we do not cherry pick a pre-trained model. Instead, in each test, we interpret a model pre-trained independently that achieves the best validation performance.

\textbf{Setup.} Since we focus on graph classification tasks, GIN ~\cite{xu2018powerful} is used as the backbone model for both baselines and \proj. We also apply PNA~\cite{corso2020principal} to further test the wide applicability of \proj, for which we adopt the no-scalars version since the scalars used in PNA are essentially a type of attention, which may conflict with our method. GIN+GSAT denotes using GIN as the base GNN encoder of GSAT, and PNA+GSAT means replacing the GNN encoder with PNA.
In addition, we apply \proj to fine-tune and interpret pre-trained models as described in Sec.~\ref{sec:finetune}, which is highlighted as $\text{\proj}^*$. 
In all the experiments, we use $r=0.7$ in Eq.~\eqref{eq:reg} by default or otherwise specified. Our studies have shown that \proj is generally robust when $r\in[0.5, 0.9]$ (see Fig.~\ref{fig:reg-compare} later).

\subsection{Result Comparison and Analysis} \label{sec:results}

\textbf{Interpretability Results.} As shown in Table \ref{table:Interpretation}, our methods significantly outperform the baselines by 9\%$\uparrow$ on average and up to 20\%$\uparrow$. If we just compare among inherently interpretable models, the boost is even more significant. Moreover, \proj also provides much stabler interpretation than the baselines as for the much smaller variance. \text{\proj}$^*$ via fine-tuning a pre-trained model can often further boost the interpretation performance. Also, when the more expressive model PNA is used as the backbone, we find the posthoc methods are likely to suffer from the overfitting issue as explained in Sec.~\ref{sec:post-hoc}. However, \proj does not suffer from that and can yield even better interpretation results. 
Over Ba-2Motifs and Mutag, GNNExplainer and PGExplainer work worse than what reported in \cite{luo2020parameterized} as we do not cherry pick the pre-trained model. However, \proj still significantly outperforms their reported performance in the Appendix \ref{appx:sup_exp}. 
We also provide visualizations of the subgraphs discovered by \proj in Appendix \ref{appx:visz}.

\textbf{Prediction Results.} As explained in Sec.~\ref{sec:better-generalization}, being trained via the GIB principle, \proj is more generalizable and thus may  achieve even better prediction performance.
As shown in Table \ref{table:Generalization}, GIN+\proj significantly outperforms the backbone GIN over the Spurious-Motif datasets, where spurious correlation  exists in the training data. For other datasets, GIN+\proj can achieve comparable results, which matches our claim that \proj provides interpretation without hurting the prediction. IB-subgraph, trained via the GIB principle, also achieves good prediction performance though its interpretability is poor (Table~\ref{table:Interpretation}).
When PNA is used, \proj improves it by about $1-5\%$ on the datasets in the first three columns. Notably, $\text{\proj}^*$ achieves the SOTA performance on \emph{molhiv} among all models that do not incorporate expert knowledge according to the \href{https://ogb.stanford.edu/docs/leader_graphprop/#ogbg-molhiv}{leaderboard}. Unexpectedly, PNA achieves very good performance on Spurious-Motif and \text{\proj}$^*$ just slightly improves it. Our results on the other $5$ molecular datasets from OGBG are showed in Table \ref{table:5mol}, where \proj and $\text{\proj}^*$ mostly outperform PNA.

\textbf{Invariant Learning Results.} We note that DIR achieves a bit lower prediction performance in Table \ref{table:Generalization} than what reported in \cite{anonymous2022discovering} even after we extensively tune its parameters, which is probably due to the different backbone models used. Hence, we also compare with DIR by using their backbone model. And we include several invariant learning baselines reported in DIR to further demonstrate the ability of \proj to remove spurious correlations. Results are shown in Table \ref{table:dir-acc}. \proj significantly outperforms all invariant learning methods on spurious correlation removal, even without utilizing causality analysis, which further validates our claims in Sec.~\ref{sec:better-generalization}.
A comparison of interpretability of these models is shown in Table \ref{table:dir-prec5} in the appendix.

\begin{table}[t]
\vspace{-2mm}
\caption{Generalization ROC AUC on other OGBG-Mol datasets. The \textbf{bold} font highlights when \proj outperforms PNA.}
\vspace{-0.1cm}
\begin{center}
\resizebox{\columnwidth}{!}{%
\begin{sc}
\begin{tabular}{lccccc}
\toprule
  & molbace           & molbbbp           & molclintox        & moltox21          & molsider          \\
\midrule
PNA & $73.52\pm3.02$ & $67.21\pm1.34$ & $86.72\pm2.33$ & $75.08\pm0.64$ & $56.51\pm1.90$ \\
\proj            & $\mathbf{77.41}\pm2.42$ & $\mathbf{69.17}\pm1.12$ & $\mathbf{87.80}\pm2.36$ & $74.96\pm0.66$ & $\mathbf{57.58}\pm1.23$ \\
$\text{\proj}^*$        & $73.61\pm1.59$ & $66.30\pm0.79$ & $\mathbf{89.26}\pm1.66$ & $\mathbf{75.71}\pm0.48$ & $\mathbf{59.19}\pm1.03$ \\
\bottomrule
\label{table:5mol}
\end{tabular}
\end{sc}
}
\end{center}
\vskip -10mm
\end{table}

\begin{table}[t]
\caption{Direct comparison (Acc.) with invariant learning methods on the ability to remove spurious correlations, by applying the backbone model used in~\cite{anonymous2022discovering}.}
\vspace{-4mm}
\begin{center}
\resizebox{\columnwidth}{!}{%
\begin{sc}
\begin{tabular}{lccc}
\toprule
Spurious-motif  & $b = 0.5$          & $b=0.7$           & $b=0.9$           \\
\midrule
ERM          & $39.69\pm1.73           $ & $38.93\pm1.74           $ & $33.61\pm1.02$            \\
V-REx        & $39.43\pm2.69           $ & $39.08\pm1.56           $ & $34.81\pm2.04$            \\
IRM          & $41.30\pm1.28           $ & $40.16\pm1.74           $ & $35.12\pm2.71$            \\
DIR          & $45.50\pm2.15           $ & $43.36\pm1.64           $ & $39.87\pm0.56$            \\
\proj        & $\mathbf{53.27}^\dagger\pm5.12           $ & $\mathbf{56.50}^\dagger\pm3.96           $ & $\mathbf{53.11}^\dagger\pm4.64$            \\
$\text{\proj}^*$    & $43.27\pm4.58           $ & $42.51\pm5.32           $ & $\mathbf{45.76}^\dagger\pm5.32$            \\
\bottomrule
\label{table:dir-acc}
\end{tabular}
\end{sc}
}
\end{center}
\vskip -0.7cm
\end{table}

\textbf{Ablation Studies.} We conduct ablation studies from three aspects: First, the importance of stochasticity in \proj, where we replace the Bernoulli sampling procedure with setting attention $\alpha_{uv}=p_{uv}$ without stochasticity; Second, the importance of the information regularization term (Eq.~\eqref{eq:reg}), where we set its coefficient $\beta=0$ in Eq.~\eqref{eq:proj}; Third, the superiority of the information regularization term over the sparsity-driven term $\ell_1$-norm. 

As shown in Table \ref{table:ab-beta-noise-gin}, the performance drops significantly when there is either no stochasticity or $\beta = 0$. Specifically, \proj-NoStoch means applying deterministic attention $\in[0,1]$, which causes the most performance drop. \proj-NoStoch-$\beta=0$ corresponds to using deterministic attention without the regularization term in Eq.~\eqref{eq:reg}, which causes the second most performance drop. \proj-$\beta=0$ denotes applying stochastic attention with no regularization, which performs better than baselines but worse than original \proj and suffers from large variance.
Overall, no stochasticity yields the biggest drop, which well matches our theory.
This also implies that directly using the deterministic attention mechanisms such as GAT~\cite{velivckovic2018graph} or GGNN~\cite{li2015gated} may not yield good interpretability. 

Fig.~\ref{fig:reg-compare} shows that our information regularization term can achieve consistently better performance than the sparsity-driven $\ell_1$-norm regularization even when the grid search is used to tune hyperparameters. We also observe that when $r$ is close to 0, the results often get decreased or have  higher variance. The best performance is often achieved when $r\in [0.5, 0.9]$, which matches our theory. More results on other datasets can be found in Fig.~\ref{fig:sp0709_l1} in the appendix.

\begin{table}[t]
\tiny
\vspace{-2mm}
\caption{Ablation study on $\beta$ and stochasticity in \proj (GIN as the backbone model) on Spurious-Motif. We report both interpretation ROC AUC (top) and prediction accuracy (bottom).}
\vspace{-0.2cm}
\begin{center}
\begin{sc}
\begin{tabular}{lccc}
\toprule
 Spurious-motif & $b = 0.5$          & $b=0.7$           & $b=0.9$           \\
\midrule
\proj                    & $79.81\pm3.98 $ & $74.07\pm5.28 $ & $71.97\pm4.41$ \\
\proj-$\beta=0$          & $66.00\pm11.04$ & $65.92\pm3.28 $ & $66.31\pm6.82$ \\
\proj-NoStoch           & $59.64\pm5.33 $ & $55.78\pm2.84 $ & $55.27\pm7.49$ \\
\proj-NoStoch-$\beta=0$ & $63.37\pm12.33$ & $60.61\pm10.08$ & $66.19\pm7.76$ \\
\midrule  
 GIN      & $39.87\pm1.30$ & $39.04\pm1.62$ & $38.57\pm2.31$ \\
\proj                    & $51.86\pm5.51$ & $49.12\pm3.29$ & $44.22\pm5.57$ \\
\proj-$\beta=0$          & $45.97\pm8.37$ & $49.67\pm7.01$ & $49.84\pm5.45$ \\
\proj-NoStoch           & $40.34\pm2.77$ & $41.90\pm3.70$ & $37.98\pm2.64$ \\
\proj-NoStoch-$\beta=0$ & $43.41\pm8.05$ & $45.88\pm9.54$ & $42.25\pm9.77$ \\
\bottomrule
\label{table:ab-beta-noise-gin}
\end{tabular}
\end{sc}
\end{center}
\vspace{-6mm}
\end{table}

\begin{figure}[t]
     \centering
     \includegraphics[trim={0.6cm 0.0cm 0.6cm 0.1cm},clip,width=0.494\linewidth]{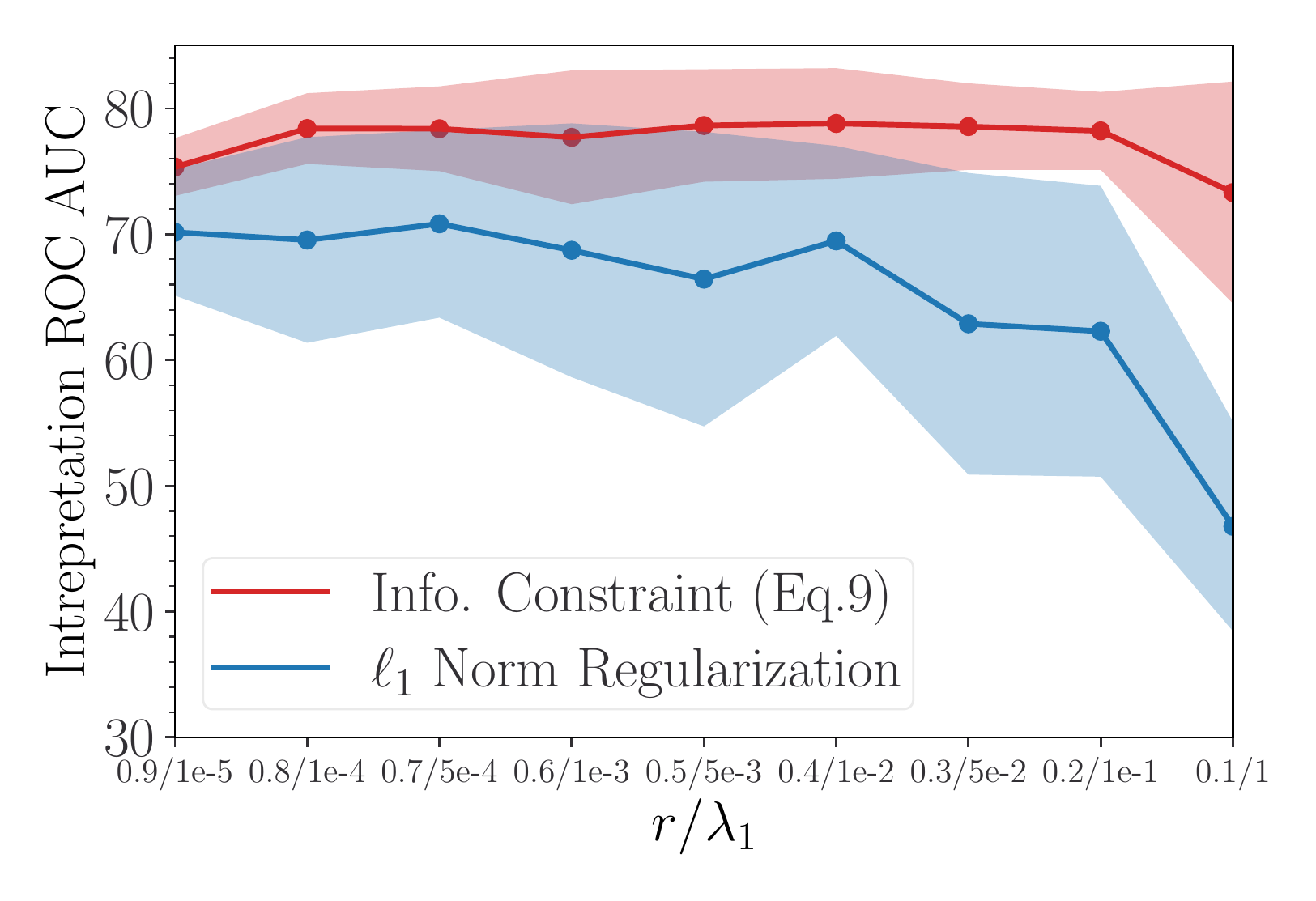}
     \includegraphics[trim={0.6cm 0.0cm 0.6cm 0.1cm},clip,width=0.494\linewidth]{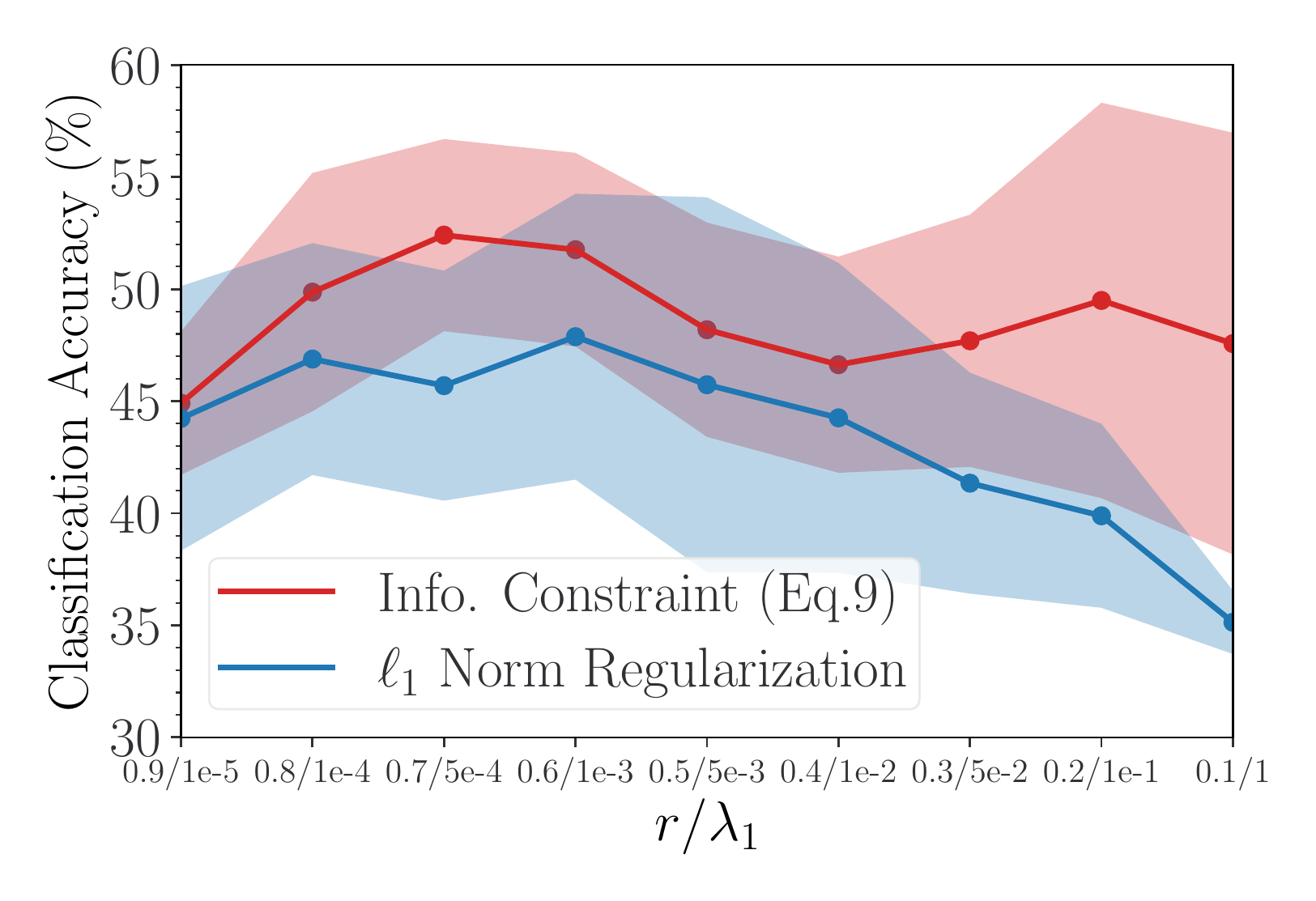}
     \vspace*{-9mm}
    \caption{Comparison between (a) using the information constraint in Eq.~\eqref{eq:reg} and (b) replacing it with $\ell_1$-norm. Results are shown for Spurious-Motif $b=0.5$, where $r$ is tuned from $0.9$ to $0.1$ and the coefficient of the $\ell_1$-norm $\lambda_1$ is tuned from $1e\text{-}5$ to $1$.}
    \vspace{-0.2cm}
    \label{fig:reg-compare}
\end{figure}

\section{Conclusion}
\emph{Graph Stochastic Attention} (\proj) is a novel attention mechanism to build interpretable graph learning models. \proj injects stochasticity to block label-irrelevant information and leverages the reduction of stochasticity to select label-relevant subgraphs. Such rationale is grounded by the information bottleneck principle. \proj has many transformative characteristics. For example, it removes the sparsity, continuity or other potentially biased assumptions in graph learning interpretation without performance decay. It can also remove spurious correlation to better the model generalization. As a by-product, we also reveal a potentially severe issue behind post-hoc interpretation methods from the optimization perspective of information bottleneck.

\subsubsection*{Acknowledgments}
We greatly thank the actionable suggestions given by reviewers. S. Miao and M. Liu are supported by the National Science Foundation (NSF) award HDR-2117997. P. Li is supported by the JPMorgan Faculty Award.

\bibliography{reference}
\bibliographystyle{icml2022}

\newpage
\appendix
\onecolumn

\section{Supplementary Notations for Information Theory and Graph Neural Networks}
\textbf{Entropy.} Given a discrete random variable $a$, its entropy is defined as $H(a) \triangleq - \sum_{a} \mathbb{P}(a)\log \mathbb{P}(a)$. If $a$ is a continuous random variable, its differential entropy is defined as $H(a) \triangleq - \int_{a} \mathbb{P}(a)\log \mathbb{P}(a)da$.  

\textbf{KL-Divergence}. Given two distributions $\mathbb{P}(x)$ and $\mathbb{Q}(x)$, KL-Divergence is used to measure the difference between $\mathbb{P}$ and $\mathbb{Q}$, and it is defined as $\text{KL}(\mathbb{P}(x)||\mathbb{Q}(x)) \triangleq \sum_{x}\mathbb{P}(x)\log \frac{\mathbb{P}(x)}{\mathbb{Q}(x)}$.

\textbf{Mutual Information.} Given two random variables $a$ and $b$, the mutual information (MI) $I(a;b)$ is a measure of the mutual dependence between them.  MI quantifies the amount of information regarding one random variable if another random variable is known. Formally, $I(a;b) \triangleq \sum_{a,b} \mathbb{P}(a,b) \log \frac{\mathbb{P}(a,b)}{\mathbb{P}(a)\mathbb{P}(b)}$, where $\mathbb{P}(a,b)$ is the joint distribution and $\mathbb{P}(a)$, $\mathbb{P}(b)$ are the marginal distributions. By definition, $I(a,b) = \text{KL}(\mathbb{P}(a,b)||\mathbb{P}(a)\mathbb{P}(b)) = \sum_{a,b}\mathbb{P}(a,b)\log\mathbb{P}(a|b) - \sum_{b}\mathbb{P}(b)\log\mathbb{P}(b)= - H(a|b) + H(b)$.

\textbf{Graph Neural Networks (GNNs).} Given an $L$-layer GNN, let $h_v^{(l)}$ denote the node representation for node $v$ in the $i^{th}$ layer  and $\mathcal{N}(v)$ denote a set of nodes adjacent to node $v$. Let $h_v^{(0)}$ be the node feature $X_v$. 
Most GNNs follow a message passing scheme, where there are two main steps in each layer: (1) neighbourhood aggregation, $m_{v}^{(l)}=\operatorname{AGG}(\{h_{u}^{(l-1)} | u \in \mathcal{N}(v)\})$; (2) node representation update, $h_{v}^{(l)}=\operatorname{UPDATE}(m_{v}^{(l)}, h_{v}^{(l-1)})$. For graph classification tasks, after obtaining $h_v^{(L)}$ for each node, the graph representation is given by $h_G=\operatorname{POOL}(\{h_v^{(L)}| v \in V\})$ and $h_G$ will be used to make predictions. The above $\operatorname{AGG}$, $\operatorname{UPDATE}$, $\operatorname{POOL}$ are three functions. $\operatorname{AGG}$ and $\operatorname{POOL}$ are typically implemented via $\operatorname{SUM}$, $\operatorname{MEAN}$ and $\operatorname{MAX}$ while $\operatorname{UPDATE}$ is a fully connected (typically shallow) neural network. In some cases, edge representations may be in need, and they are often given by $h_{u,v}^{(l)} = \operatorname{CONCAT}(h_u^{(l)}, h_v^{(l)})$.

\section{Variational Bounds for the GIB Objective --- Eq.~\eqref{eq:predictor} and Eq.~\eqref{eq:extractor}} \label{appx:deriving}
From Eq.~\eqref{eq:GIB2}, the IB objective is:
\begin{align} 
    \min_{\phi} -I(G_S ; Y) + \beta I(G_S;G),\, \text{s.t.}\,\, G_S \sim g_{\phi}(G). 
\end{align}
To optimize it, we introduce two variational bounds on the two terms, respectively.

For the first term $I\left(G_S ; Y\right)$, by definition:
\begin{equation}
    I\left(G_S ; Y\right) = \mathbb{E}_{G_S, Y} \left[ \log \frac{\mathbb{P}(Y|G_S)}{\mathbb{P}(Y)} \right].
\end{equation}
Since $\mathbb{P}(Y|G_S)$ is intractable, we introduce a variational approximation $\mathbb{P}_{\theta}(Y|G_S)$ for it. Then, we obtain a lower bound for Eq.~\eqref{eq:predictor}:
\begin{align}  
    I\left(G_S ; Y\right) &= \mathbb{E}_{G_S, Y} \left[ \log \frac{\mathbb{P}_{\theta}(Y|G_S)}{ \mathbb{P}(Y)} \right] + \mathbb{E}_{G_S} \left[ \text{KL}( \mathbb{P}(Y|G_S)  || \mathbb{P}_{\theta}(Y|G_S)) \right] \nonumber  \\
    & \geq \mathbb{E}_{G_S, Y} \left[ \log \frac{\mathbb{P}_{\theta}(Y|G_S)}{ \mathbb{P}(Y)} \right] \nonumber  \\
    & = \mathbb{E}_{G_S, Y} \left[ \log {\mathbb{P}_{\theta}(Y|G_S)} \right] + H(Y).
\end{align}

For the second term $I \left(G ; G_S\right)$, by definition:
\begin{equation}
    I \left(G ; G_S\right) = \mathbb{E}_{G_S, G} \left[ \log \frac{\mathbb{P}(G_S|G)}{\mathbb{P}(G_S)} \right].
\end{equation}
Since $\mathbb{P}(G_S)$ is intractable, we introduce a variational approximation $\mathbb{Q}(G_S)$ for the marginal distribution $\mathbb{P}(G_S) = \sum_{G}\mathbb{P}_{\phi}(G_S|G)\mathbb{P}_{\mathcal{G}}(G)$. Then, we obtain an upper bound for Eq.~\eqref{eq:extractor}:
\begin{align}
    I \left(G ; G_S\right) &= \mathbb{E}_{G_S, G} \left[ \log \frac{\mathbb{P}_{\phi}(G_S|G)}{\mathbb{Q}(G_S)} \right] - \text{KL} \left( \mathbb{P}(G_S) || \mathbb{Q}(G_S) \right) \nonumber \\
    &\leq \mathbb{E}_{G} \left[ \text{KL} \left( \mathbb{P}_{\phi}(G_S|G) || \mathbb{Q}(G_S) \right) \right].
\end{align}

\section{Supplementary Experiments} \label{appx:setting}

\subsection{Details of the Datasets} \label{appx:datasets}
\textbf{Mutag}~\cite{debnath1991structure} is a molecular property prediction dataset, where nodes are atoms and edges are chemical bonds. Each graph is associated with a binary label based on its mutagenic effect. Following~\cite{luo2020parameterized}, -NO$_2$ and -NH$_2$ in mutagen graphs are labeled as ground-truth explanations.

\textbf{BA-2Motifs}~\cite{luo2020parameterized} is a synthetic dataset, where the base graph is generated by Barabási-Albert (BA) model. Each base graph is attached with a house-like motif or a five-node cycle motif. House motifs and cycle motifs give class labels and thus are regarded as ground-truth explanations for the two classes respectively.

\textbf{Spurious-Motif}~\cite{anonymous2022discovering} is a synthetic dataset with three graph classes. Following the notations in~\cite{anonymous2022discovering}, each graph consists of a base graph (tree/ladder/wheel denoted by $\bar{G}_S = 0, 1, 2$ respectively, with some abuse of notations) and a motif (cycle/house/crane denoted by $G_S = 0, 1, 2$, respectively, with some abuse of notations). The label is determined only by $G_S$, while there also exists spurious correlation between the label and $\bar{G}_S$. Specifically, to construct a graph in the training set, $G_S$ will be sampled uniformly, while $\bar{G}_S$ will be sampled with probability $\mathbb{P}(\bar{G}_S)$, where $\mathbb{P}(\bar{G}_S) = b$ if $\bar{G}_S=G_S$; otherwise $\mathbb{P}(\bar{G}_S) = (1-b) / 2$. So, $b$ is a parameter used to control the degree of such spurious correlation. When $b=1/3$, there is no spurious correlation. We include datasets with $b=0.5$, $b=0.7$ and $b=0.9$. Note that for testing data, the motifs and bases are randomly attached to each other, which can test if the model overfits the spurious correlation.

\textbf{MNIST-75sp} ~\cite{knyazev2019understanding} is a image classification dataset, where each image in MNIST is converted to a superpixel graph. Each node in the graph represents a superpixel and edges are formed based on spatial distance between superpixel centers. Node features are the coordinates of their centers of masses. Nodes with nonzero pixel values provide ground-truth explanations. Note that the subgraphs that provide explanations are of different sizes in this dataset.
    
\textbf{Graph-SST2} ~\cite{socher2013recursive, yuan2020explainability} is a sentiment analysis dataset, where each text sequence in SST2 is converted to a graph. Each node in the graph represents a word and edges are formed based on relationships between different words. We follow the dataset splits in~\cite{anonymous2022discovering} to create degree shifts in the training set, which can better test generalizability of models. Specifically, graphs with higher average node degree will be used to train and validate models, while graphs with fewer nodes will be used to test models. And this dataset contains no ground-truth explanation labels, so we only evaluate prediction performance here and provide interpretation visualizations in Appendix~\ref{appx:visz}.

\textbf{OGBG-Molhiv} ~\cite{wu2018moleculenet, hu2020ogb} is a molecular property prediction datasets, where nodes are atoms and edges are chemical bonds. A binary label is assigned to each graph according to whether a molecule inhibits HIV virus replication or not. We also evaluate \proj on molbace, molbbbp, molclintox, moltox21 and molsider datasets from OGBG. As there are no ground truth explanation labels for these datasets, we only evaluate the prediction performance of \proj.  

\begin{table}[t]
\caption{Direct comparison with the interpretation ROC AUC of GNNExplainer and PGExplainer reported in~\cite{luo2020parameterized}, which are given a selected pre-trained model.}
\begin{center}
\begin{sc}
\begin{tabular}{lcc}
\toprule
  & Ba-2motifs     & Mutag          \\
\midrule
GNNExplainer & $74.2$ & $72.7$  \\
PGExplainer & $92.6$ & $87.3$       \\
\proj   & $\mathbf{98.74}^\dagger\pm0.55  $ & $\mathbf{99.60}^\dagger\pm0.51$ \\
$\text{\proj}^*$ & $\mathbf{97.43}^\dagger\pm0.02$ & $\mathbf{97.75}^\dagger\pm0.92$ \\
\bottomrule
\label{table:pge-auc}
\end{tabular}
\end{sc}
\end{center}
\vspace{-9mm}
\end{table}

\begin{table}[t]
\caption{Direct comparison with the interpretation precision@5 of DIR reported in \cite{anonymous2022discovering} based on the backbone model in \cite{anonymous2022discovering}.}
\begin{center}
\begin{sc}
\begin{tabular}{lccc}
\toprule
  & \multicolumn{3}{c}{Spurious-motif}                   \\
  & $b = 0.5$          & $b=0.7$           & $b=0.9$           \\
\midrule
GNNExplainer & $0.203\pm0.019$ & $0.167\pm0.039$ & $0.066\pm0.007$ \\
DIR & $0.255\pm0.016$ & $0.247\pm0.012$ & $0.192\pm0.044$ \\
\proj & $\mathbf{0.519}^\dagger\pm0.022$ & $\mathbf{0.503}^\dagger\pm0.034$ & $\mathbf{0.416}^\dagger\pm0.081$ \\
$\text{\proj}^*$ & $\mathbf{0.532}^\dagger\pm0.019$ & $\mathbf{0.512}^\dagger\pm0.011$ & $\mathbf{0.520}^\dagger\pm0.022$ \\
\bottomrule
\label{table:dir-prec5}
\end{tabular}
\end{sc}
\end{center}
\vspace{-9mm}
\end{table}

\begin{table}[t]
\caption{Ablation study on $\beta$ and stochasticity in GSAT (PNA as the
backbone model) on Spurious-Motif. We report both interpretation
ROC AUC (top) and prediction accuracy (bottom).}
\begin{center}
\begin{sc}
\begin{tabular}{lccc}
\toprule
Spurious-motif  & $b = 0.5$          & $b=0.7$           & $b=0.9$           \\
\midrule
PNA+\proj                    & $83.34\pm2.17 $ & $86.94\pm4.05 $ & $88.66\pm2.44$  \\
PNA+\proj-$\beta=0$          & $82.01\pm6.43 $ & $78.88\pm6.74 $ & $80.53\pm5.03$  \\
PNA+\proj-NoStoch           & $79.72\pm3.86 $ & $76.36\pm2.57 $ & $80.21\pm3.76$  \\
PNA+\proj-NoStoch-$\beta=0$ & $78.69\pm10.77$ & $78.97\pm13.95$ & $79.91\pm13.11$ \\
\midrule
PNA                          & $68.15\pm2.39 $ & $66.35\pm3.34 $ & $61.40\pm3.56$ \\
PNA+\proj                    & $68.74\pm2.24 $ & $64.38\pm3.20 $ & $57.01\pm2.95$ \\
PNA+\proj-$\beta=0$          & $59.68\pm7.28 $ & $58.03\pm11.84$ & $53.94\pm8.11$ \\
PNA+\proj-NoStoch.           & $51.92\pm11.17$ & $41.22\pm7.72 $ & $39.56\pm2.74$ \\
PNA+\proj-NoStoch.-$\beta=0$ & $56.54\pm6.88 $ & $48.93\pm10.33$ & $45.82\pm9.60$ \\
\bottomrule
\label{table:ab-beta-noise-pna}
\end{tabular}
\end{sc}
\end{center}
\vspace{-9mm}
\end{table}

\begin{figure}[t]
    \vskip 0.1in
     \centering
     \begin{subfigure}[t]{0.49\linewidth}
         \centering
         \includegraphics[width=0.49\linewidth]{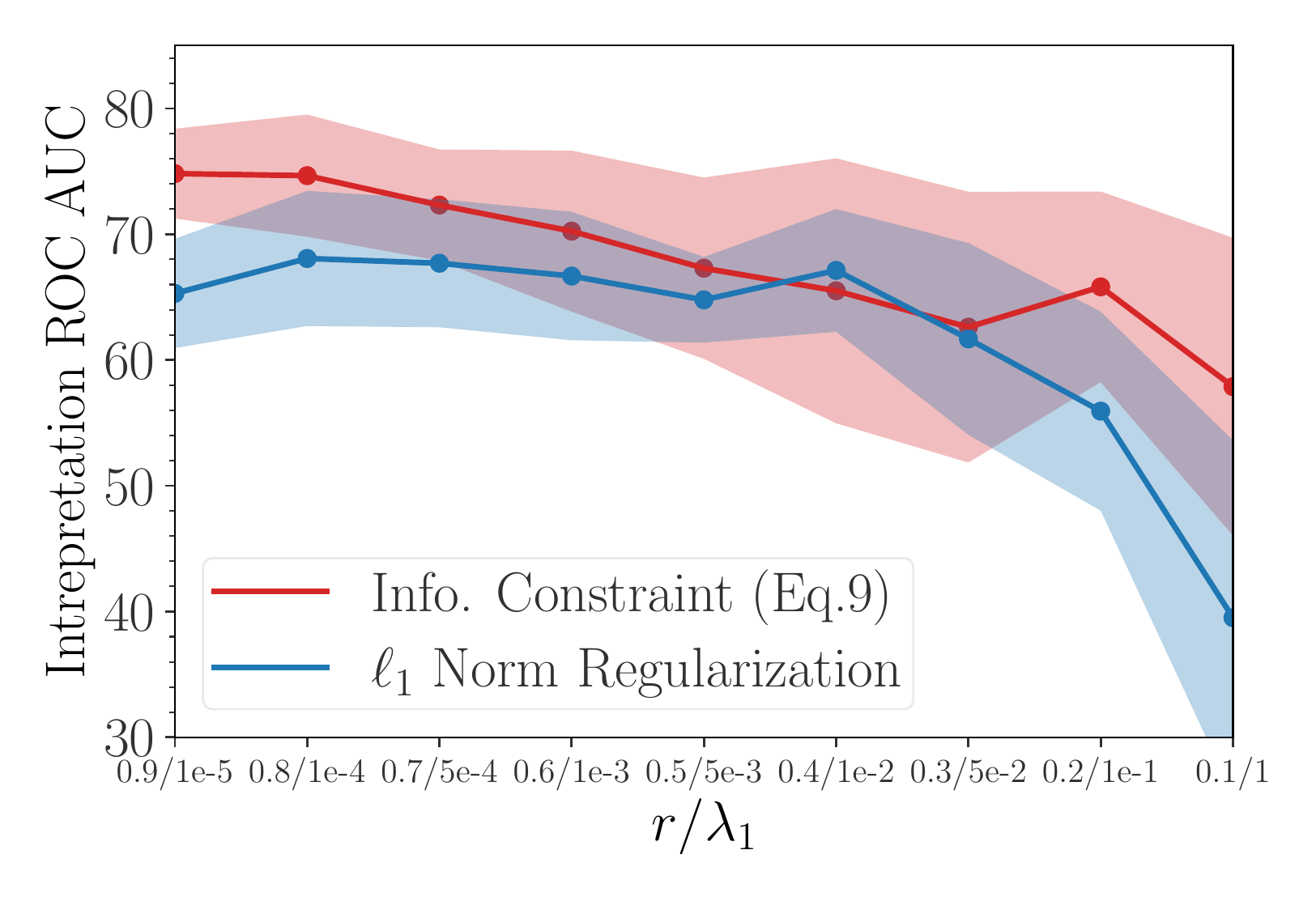}
         \includegraphics[width=0.49\linewidth]{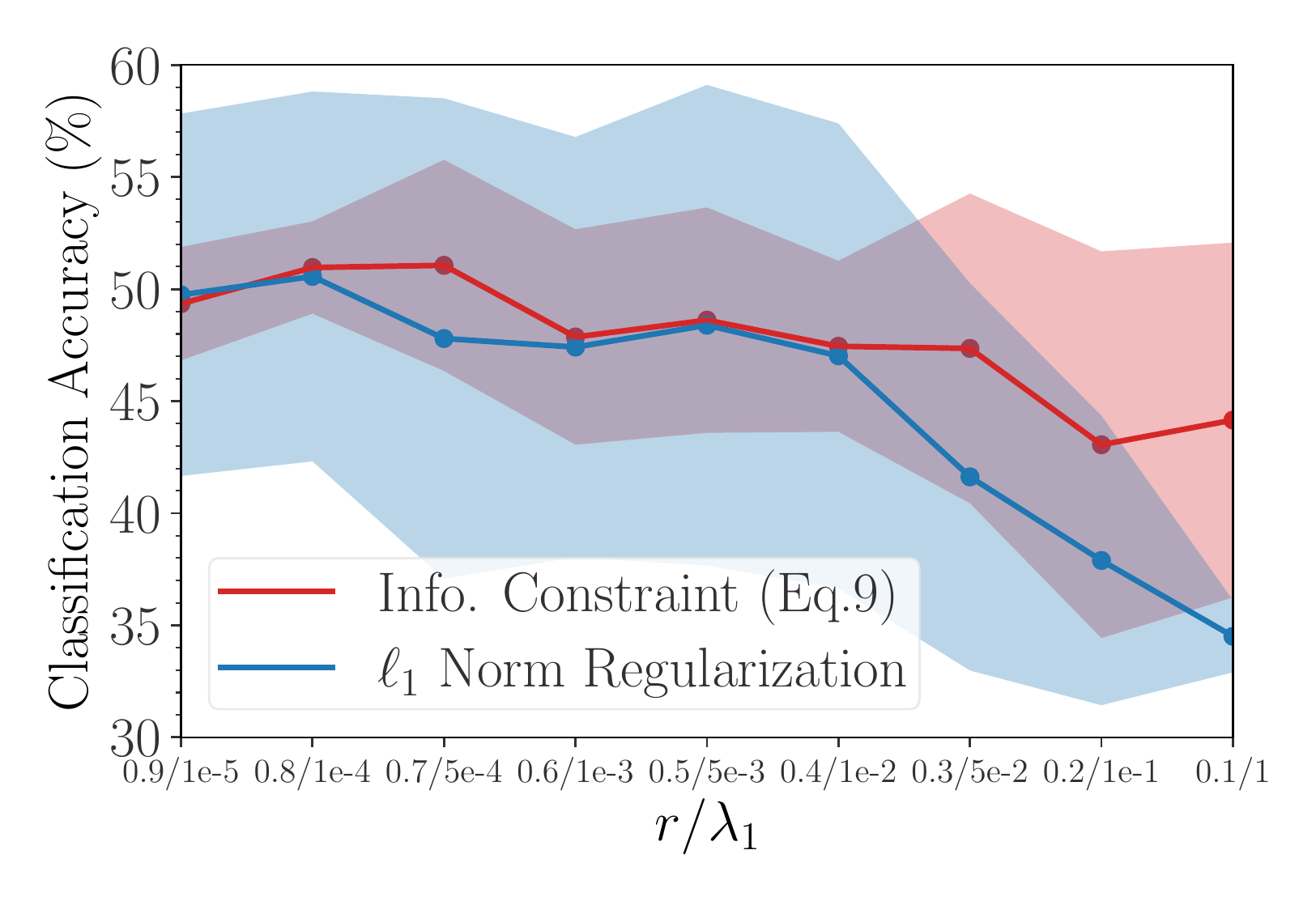}
         \vspace*{-4mm}
         \caption{Spurious-Motif, $b=0.7$}
     \end{subfigure}
     \hfill
     \begin{subfigure}[t]{0.49\linewidth}
         \centering
         \includegraphics[width=0.49\linewidth]{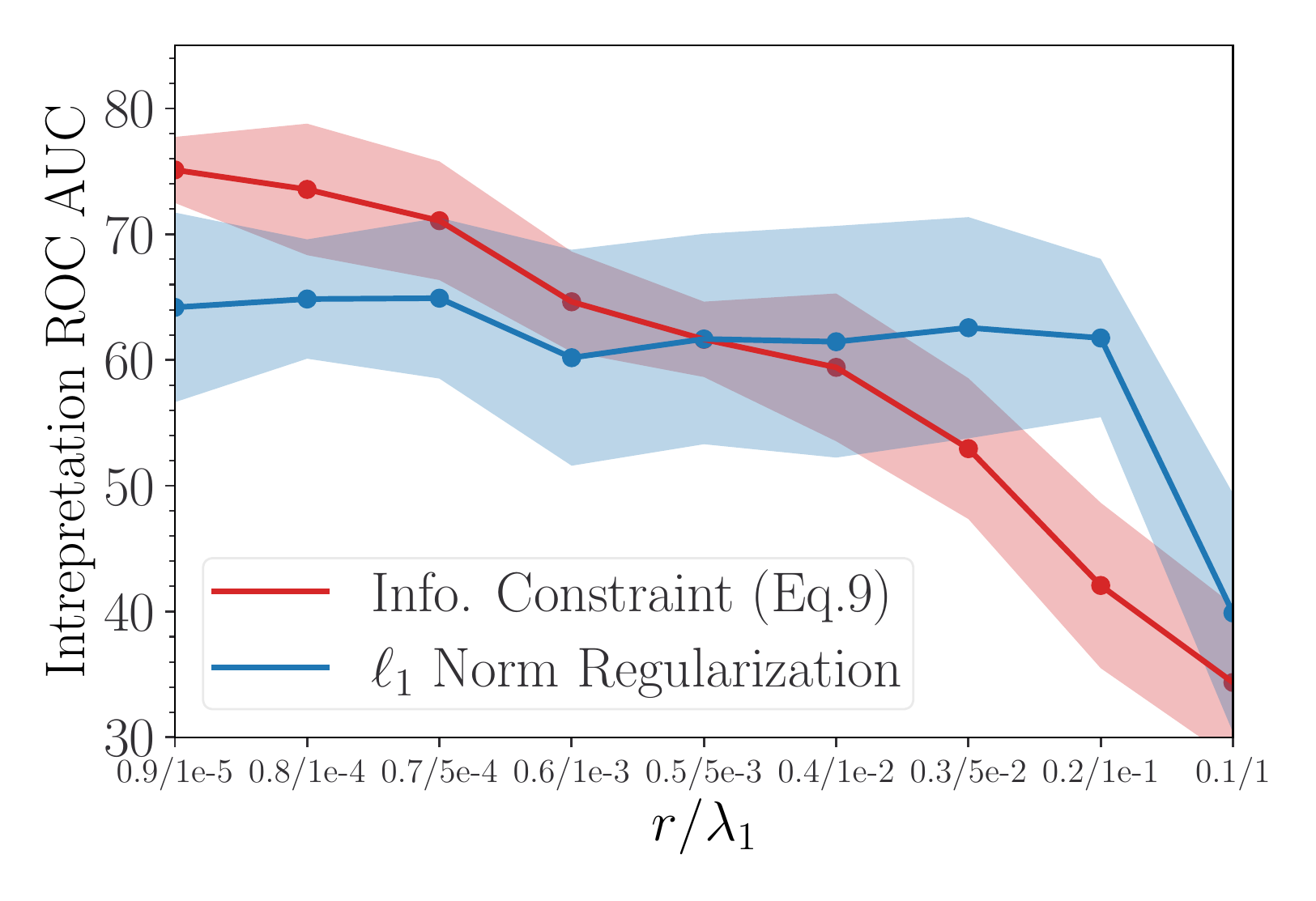}
         \vspace*{-4mm}
         \includegraphics[width=0.49\linewidth]{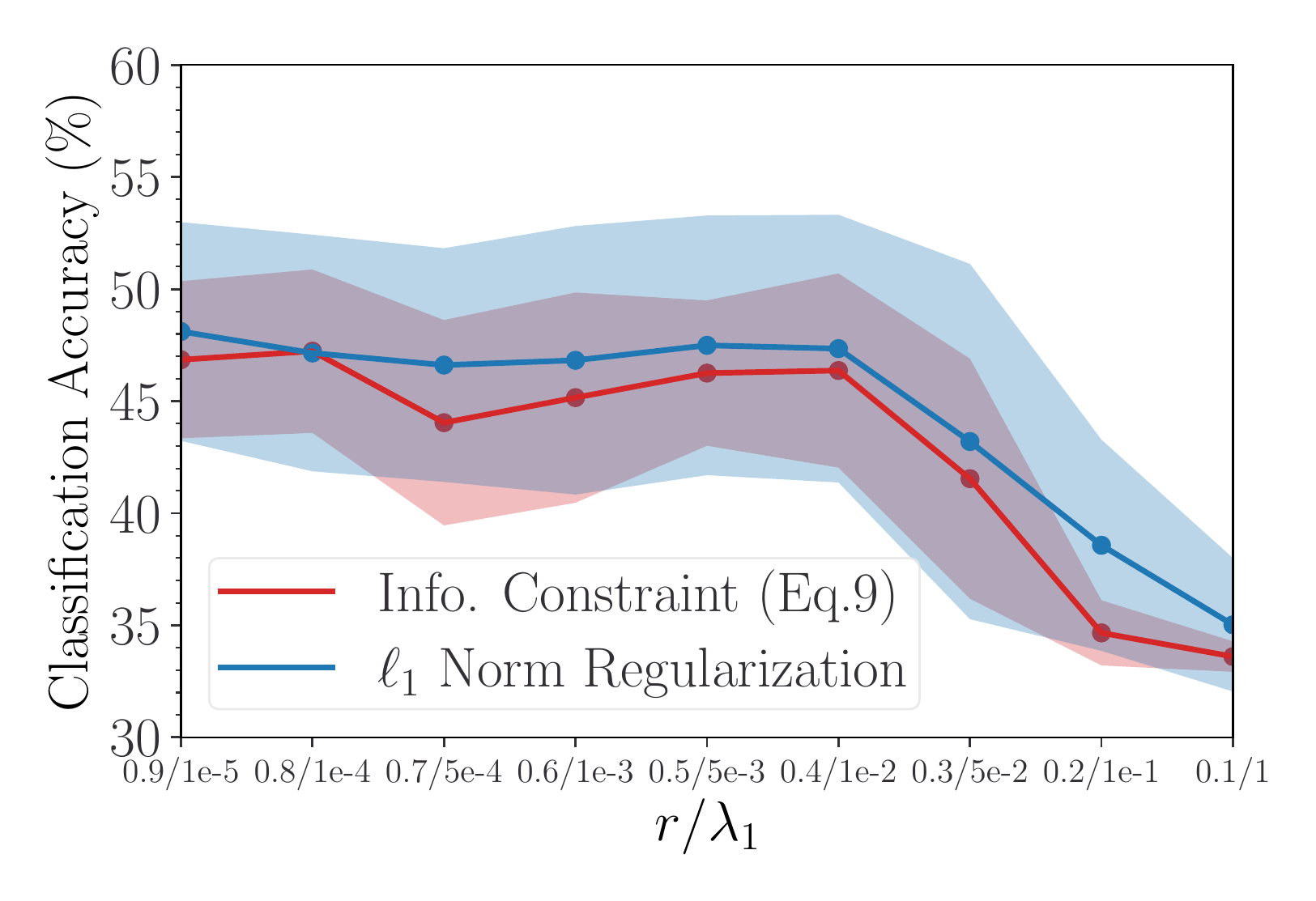}
         \caption{Spurious-Motif, $b=0.9$}
     \end{subfigure}
     \vspace*{-2mm}
    \caption{Ablation study on (a) using the info. constraint in Eq.~\eqref{eq:reg} and (b) replacing it with $\ell_1$-norm, where $r$ is tuned from $0.9$ to $0.1$ and the coefficient of the $\ell_1$-norm $\lambda_1$ is tuned from $1e\text{-}5$ to $1$.}
    \label{fig:sp0709_l1}
       \vspace*{-2mm}
\end{figure}

\subsection{Details on Hyperparameter Tuning}

\subsubsection{Backbone Models}
\textbf{Backbone Architecture.} We use a two-layer GIN ~\cite{xu2018powerful} with $64$ hidden dimensions and $0.3$ dropout ratio. We use the setting from ~\cite{corso2020principal} for PNA, which has 4 layers with $80$ hidden dimensions, $0.3$ dropout ratio, and no scalars are used. For OGBG-Mol datasets, we directly follow ~\cite{corso2020principal} using (mean, min, max, std) aggregators for PNA; yet we find PNA has convergence issues on other datasets when sum aggregator is not used. Hence, PNA uses (mean, min, max, std, sum) aggregators for all other datasets.

\textbf{Dataset Splits.} For Ba-2Motifs, we split it randomly into three sets ($80\%/10\%/10\%$). For Mutag, we split it randomly into $80\%/20\%$ to train and validate models, and following~\cite{luo2020parameterized} we use mutagen molecules with -NO$_2$ or -NH$_2$ as test data (because only these samples have explanation labels). For MNIST-75sp, we use the default splits given by \cite{knyazev2019understanding}; due to its large size in the graph setting, we also reduce the number of training samples following~\cite{anonymous2022discovering} to speed up training. For Graph-SST2, Spurious-Motifs and OGBG-Mol, we use the default splits given by \cite{yuan2020explainability} and \cite{anonymous2022discovering}. Following~\cite{corso2020principal}, edge features are not used for all OGBG-Mol datasets.

\textbf{Epoch.} We tune the number of epochs to make sure the convergence of all models. When GIN is used as the backbone model, MNIST-75sp and OGBG-Molhiv are trained for $200$ epochs, and all other datasets are trained for $100$ epochs. When PNA is used, Mutag and Ba-2Motifs are trained for $50$ epochs and all other datasets are trained for $200$ epochs. We report the performance of the epoch that achieves the best validation prediction performance and use the models that achieve such best validation performance as the pre-trained models. When multiple epochs achieve the same best performance, we report the one with the lowest validation prediction loss.

\textbf{Batch Size.} All datasets use a batch size of $128$; except for MNIST-75sp we use a batch size of $256$ to speed up training due to its large size in the graph setting.

\textbf{Learning Rate.} GIN uses $0.003$ learning rate for Spurious-Motifs and $0.001$ for all other datasets. PNA uses $0.01$ learning rate with scheduler following~\cite{corso2020principal}, $0.003$ learning rate for Graph-SST2 and Spurious-Motifs, and $0.001$ learning rate for all other datasets. 

\subsubsection{\proj}

\textbf{Basic Setting.} If not specified, \proj uses the same settings mentioned for the backbone models. All Spurious-Motif datasets share the same hyperparameters, which are tuned based on $b=0.5$.

\textbf{Learning Rate.} When PNA is used, \proj uses $0.001$ learning rate for all OGBG-Mol datasets; otherwise it uses the same learning rate as mentioned above.

\textbf{$r$ in Equation (\ref{eq:reg}).} Ba-2Motif and Mutag use $r=0.5$, and all other datasets use $r=0.7$. We find $r=0.7$ can generally provide great performance for all datasets. Inspired by curriculum learning~\cite{bengio2009curriculum}, $r$ will initially set to $0.9$ and gradually decay to the tuned value. We adopt a step decay, where $r$ will decay $0.1$ for every $10$ epochs.

\textbf{$\beta$ in Equation (\ref{eq:proj}).} $\beta$ is not tuned and is set to $\frac{1}{|E|}$ for all datasets.

\textbf{Temperature.} Temperature used in the Gumbel-softmax trick~\cite{jang2016categorical} is not tuned, and we use $1$ for all datasets.

\subsubsection{Baseline Interpretable Methods/Models}

\textbf{Basic Setting.} If not specified, baselines use the same settings mentioned for the backbone models. All Spurious-Motif datasets share the same hyperparameters, which are tuned based on $b=0.5$.

\textbf{GNNExplainer.} We tune the learning rate from $(1, 0.1, 0.01, 0.001)$ and the coefficient of the $\ell_1$-norm from $(0.1, 0.01, 0.001)$, based on validation interpretation ROC AUC. The coefficient of the entropy regularization term is set to the recommended value $1$. Again, in a real-world setting, post-hoc methods have no clear metric to tune hyper-parameters.

\textbf{PGExplainer.} We use the tuned recommended settings from~\cite{luo2020parameterized}, including the temperature, the coefficient of $\ell_1$-norm regularization and the coefficient of entropy regularization. 

\textbf{GraphMask.} We use the recommended settings from~\cite{schlichtkrull2021interpreting}, including the temperature, gamma, zeta and the coefficient of $\ell_0$-norm regularization. 

\textbf{DIR.} Causal ratio is tuned for Ba-2Motif and Mutag. Since the other datasets we use are the same, we use the recommended settings from~\cite{anonymous2022discovering}. However, even though datasets are the same, we find the same $\alpha$ specified in their source code do not work well in our setting. Hence, we tune $\alpha$ from $(10, 1, 0.1, 0.01, 0.001, 0.0001, 0.00001, 0.000001)$. 

\textbf{IB-subgraph.} Due to the extreme inefficiency of IB-subgraph, we are only able to tune its mi-weight around the recommended value from $(2, 0.2, 0.02)$. And we use the default inner loop iterations and con-weight as specified in their source code. IB-subgraph needs $\sim$40 hours to train $100$ epochs for $1$ seed on Spurious-Motif and $\sim$150 hours for OGBG-Molhiv on a Quadro RTX 6000. By contrast, \proj only needs $\sim$15 minutes to train $100$ epochs on OGBG-Molhiv.

\textbf{Random Seed.} All methods are trained with $10$ different random seeds; except for IB-subgraph we train it for $5$ different random seeds due to its inefficiency. For post-hoc methods, the pre-trained models are also trained with $10$ different random seeds instead of a fixed pre-trained model in~\cite{luo2020parameterized}. For inherently interpretable models, \proj, IB-subgraph and DIR, we average the best epoch's performance according to their validation prediction performance. For post-hoc baselines, we average their last epoch's performance. For IB-subgraph, we stop training when there is no improvement over $20$ epochs to make the training possible on large datasets.

\subsection{Node/Edge Attention}
We also explore node-level attention, and we find it is especially useful for molecular datasets and datasets with large graph sizes. Hence, we use node-level attention for on Mutag, MNIST-75sp and OGBG-Mol datasets, and for all other datasets we use edge attention. Specifically, when node attention is used, the MLP layers in $\mathbb{P}_{\phi}$ will take as input the node embeddings and output $p_v$ for each $v \in V$. Then, the stochastic node attention is sampled for each node $\alpha_v \sim \text{Bern}(p_v)$. After that, $\alpha_{uv}$ is obtained by $\alpha_{uv} = \alpha_u \alpha_v$.

\subsection{Further Supplementary Experiments} \label{appx:sup_exp}
Fig.~\ref{fig:connect-visual} shows an experiment with disconnected critical subgraphs, where the dataset is generated in a similar way used to generate Ba-2Motifs. Specifically, each base graph is generated using the BA model and will be attached with two house motifs or three house motifs randomly. The number of house motifs represents the graph class. Both \proj and GraphMask are trained with the same settings used on Ba-2Motifs.

Table~\ref{table:pge-auc} shows a direct comparison with PGExplainer and GNNExplainer between the interpretation ROC AUC reported in~\cite{luo2020parameterized} and the performance of \proj. And \proj still outperforms their methods significantly.

Table~\ref{table:dir-acc} and Table~\ref{table:dir-prec5} show direct comparisons with DIR, where we apply \proj with the backbone model used in DIR. And \proj still greatly outperforms their method.

Table~\ref{table:ab-beta-noise-pna} shows the ablation study on $\beta$ and stochasticity in \proj, where PNA is the backbone model. Figure~\ref{fig:sp0709_l1} shows the ablation study of the information constraint introduced in Eq.~\eqref{eq:reg} on Spurious-Motif $b=0.7$ and $b=0.9$. We observe the same trends from these ablation studies as discussed in Sec.~\ref{sec:results}.

\section{Interpretation Visualization} \label{appx:visz}
We provide visualizations of the label-relevant subgraphs discovered by \proj on eight datasets, as shown from Fig.~\ref{viz:mutag} to Fig.~\ref{viz:mnist}. The transparency of the edges shown in the figures represents the normalized attention weights learned by \proj. The normalized attention weights are to rescale the learnt weights $\{p_{uv}|(u,v)\in E\}$ to $[0,1]$: For each graph, denote $p_{\min} = \min \{p_{uv}|(u,v)\in E\}$  and $p_{\max} = \max \{p_{uv}|(u,v)\in E\}$. We rescale the weights according to
\begin{align}\label{eq:normalization}
    \hat{p}_{uv} = \frac{p_{uv}-p_{\min}}{p_{\max}-p_{\min}}
\end{align}  

\begin{figure}[t]
\begin{center}
\centerline{\includegraphics[width=1\columnwidth]{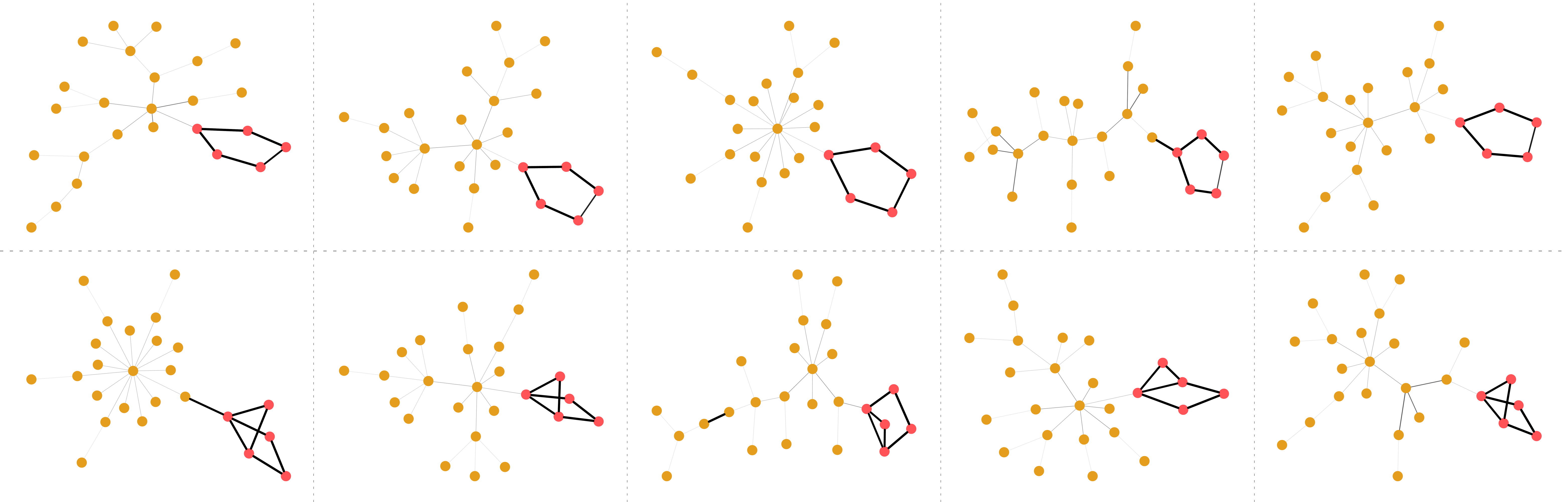}}
\vspace{-3mm}
\caption{Visualizing label-relevant subgraphs discovered by \proj for Ba-2Motifs. Nodes colored pink are ground-truth explanations, and each row represents a graph class.}
\label{viz:mutag}
\end{center}
\vspace{-5mm}
\end{figure}

\begin{figure}[t]
\begin{center}
\centerline{\includegraphics[width=1\columnwidth]{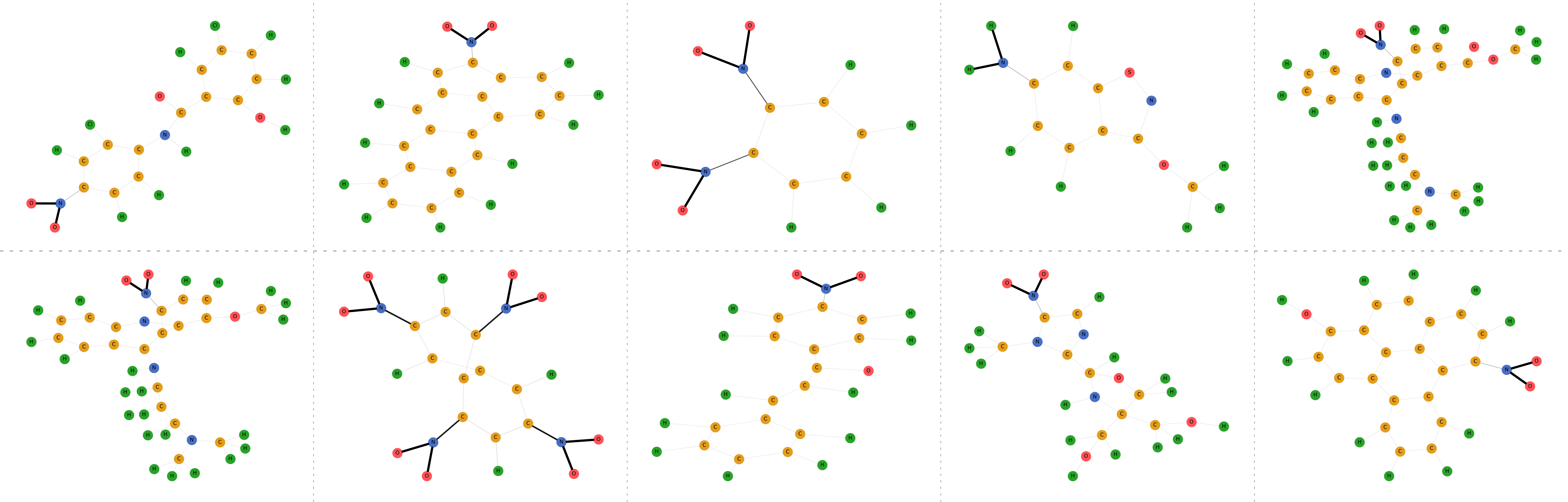}}
\vspace{-3mm}
\caption{Visualizing label-relevant subgraphs discovered by \proj for Mutag. -NO$_2$ and -NH$_2$ are ground-truth explanations. We only present mutagen graphs as only these graphs are with ground-truth explanation labels.}
\end{center}
\vspace{-7mm}
\end{figure}

\begin{figure}[t]
\begin{center}
\centerline{\includegraphics[width=1\columnwidth]{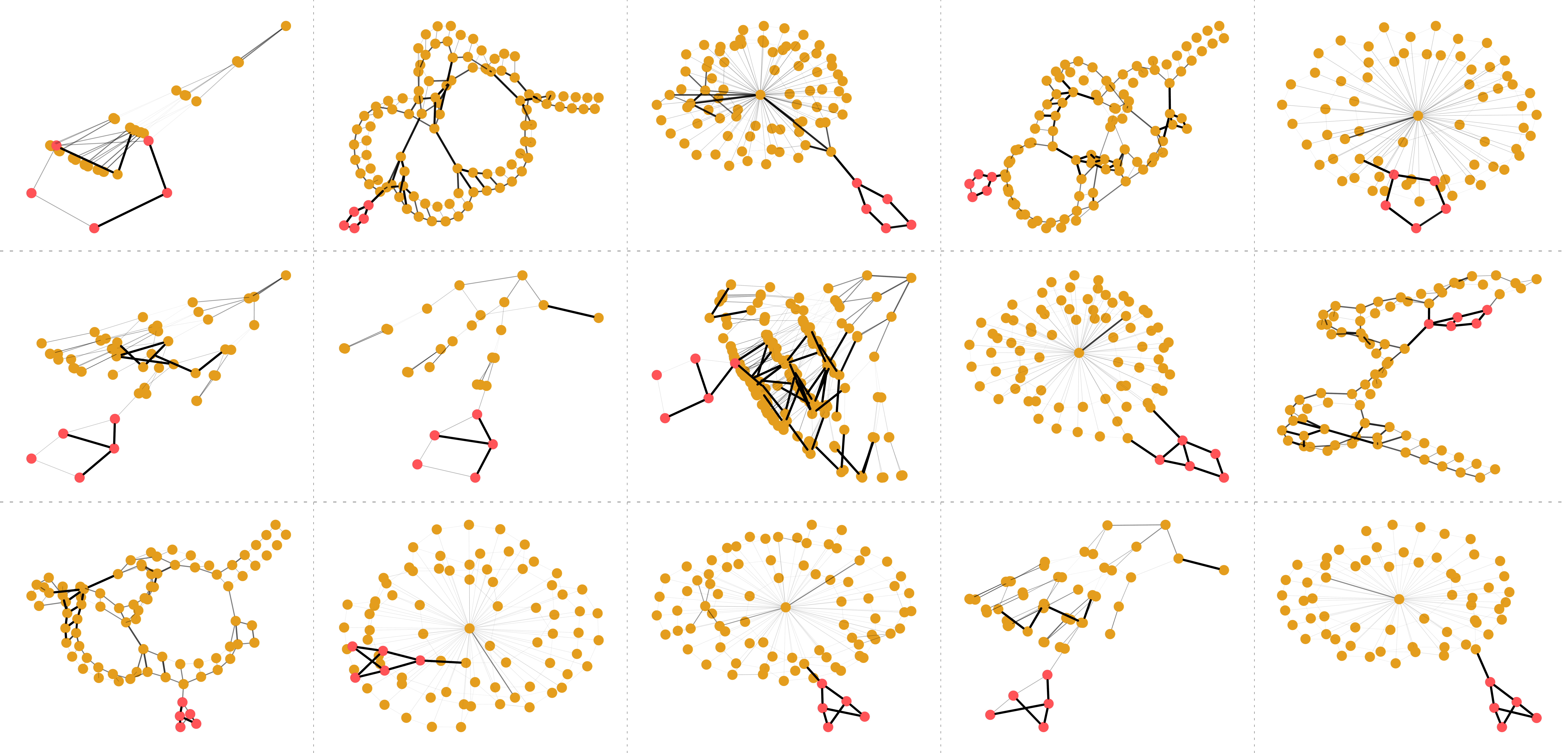}}
\caption{Visualizing label-relevant subgraphs discovered by \proj for Spurious-Motif $b=0.5$. Nodes colored pink are ground-truth explanations, and each row represents a graph class.}
\end{center}
\end{figure}

\begin{figure}[t]
\begin{center}
\centerline{\includegraphics[width=1\columnwidth]{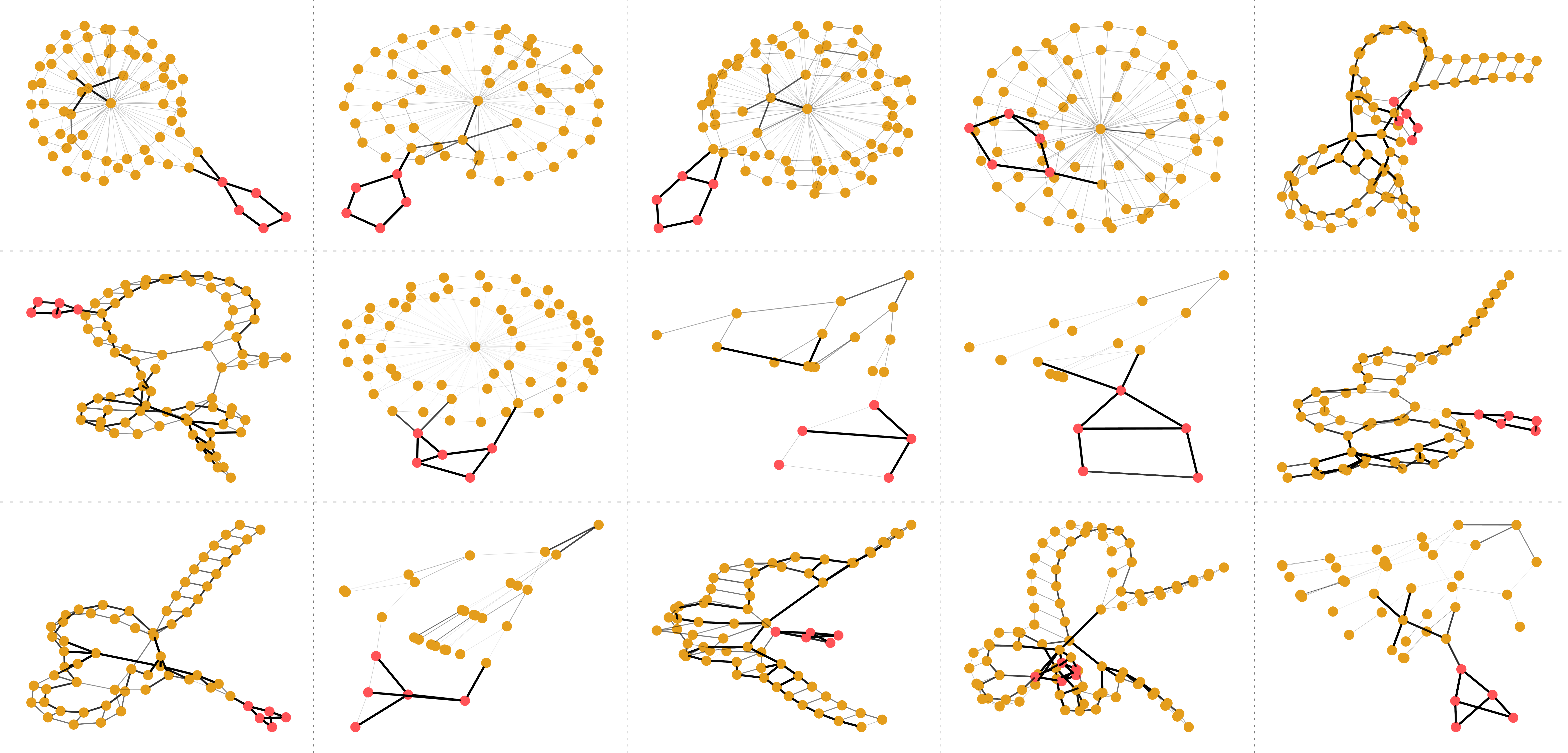}}
\caption{Visualizing label-relevant subgraphs discovered by \proj for Spurious-Motif $b=0.7$. Nodes colored pink are ground-truth explanations, and each row represents a graph class.}
\end{center}
\end{figure}

\begin{figure}[t]
\begin{center}
\centerline{\includegraphics[width=1\columnwidth]{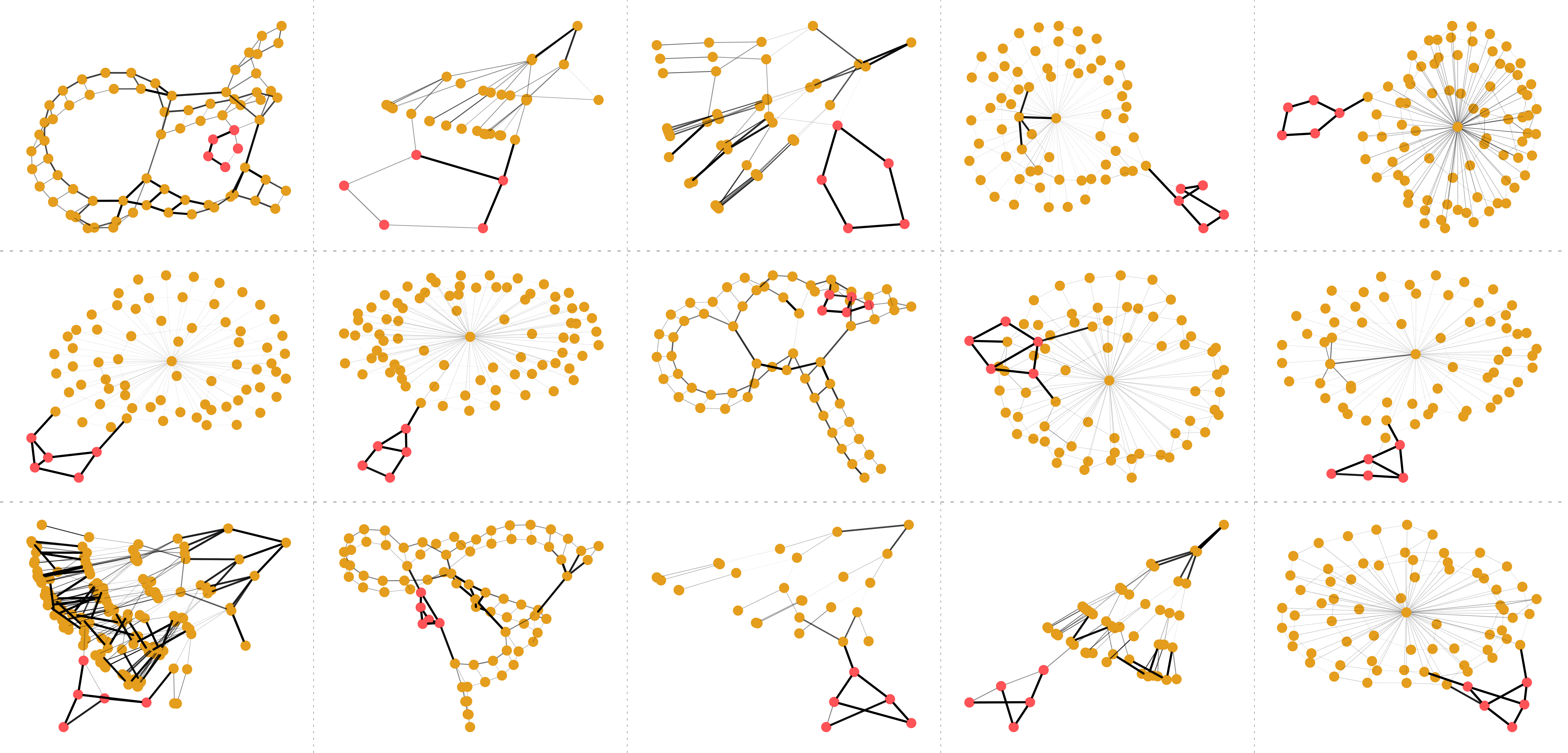}}
\caption{Visualizing label-relevant subgraphs discovered by \proj for Spurious-Motif $b=0.9$. Nodes colored pink are ground-truth explanations, and each row represents a graph class.}
\end{center}
\end{figure}

\begin{figure}[t]
\begin{center}
\centerline{\includegraphics[width=1\columnwidth]{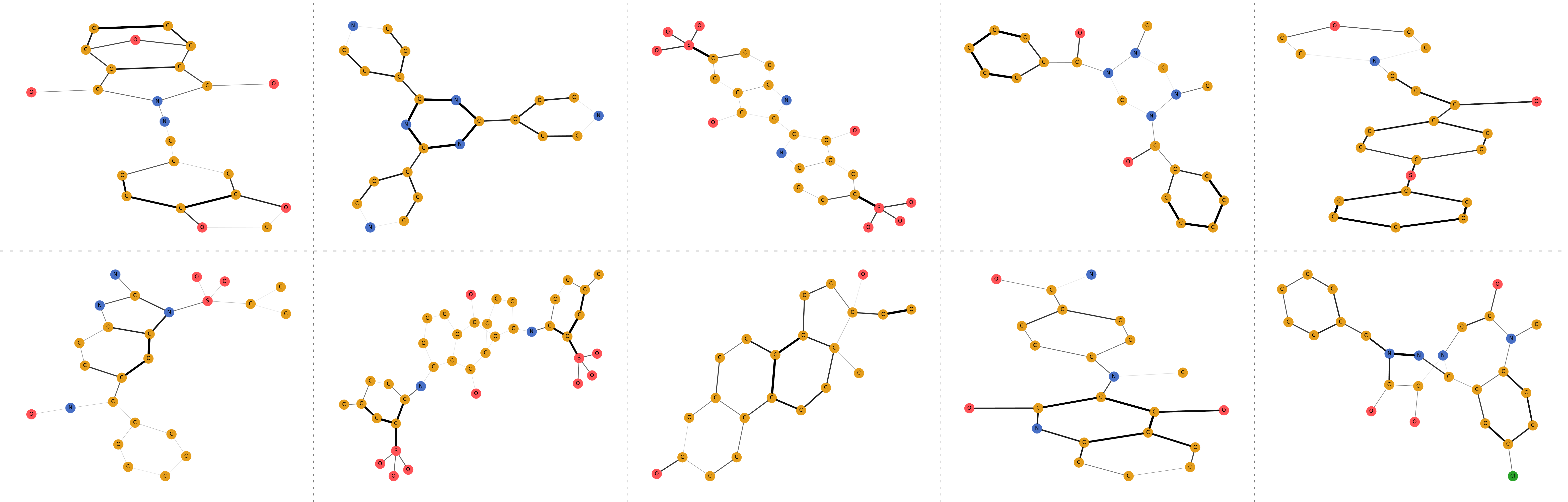}}
\caption{Visualizing label-relevant subgraphs discovered by \proj for OGBG-Molhiv. Each row represents a graph class.}
\end{center}
\end{figure}

\begin{figure}[t]
\begin{center}
\centerline{\includegraphics[width=1\columnwidth]{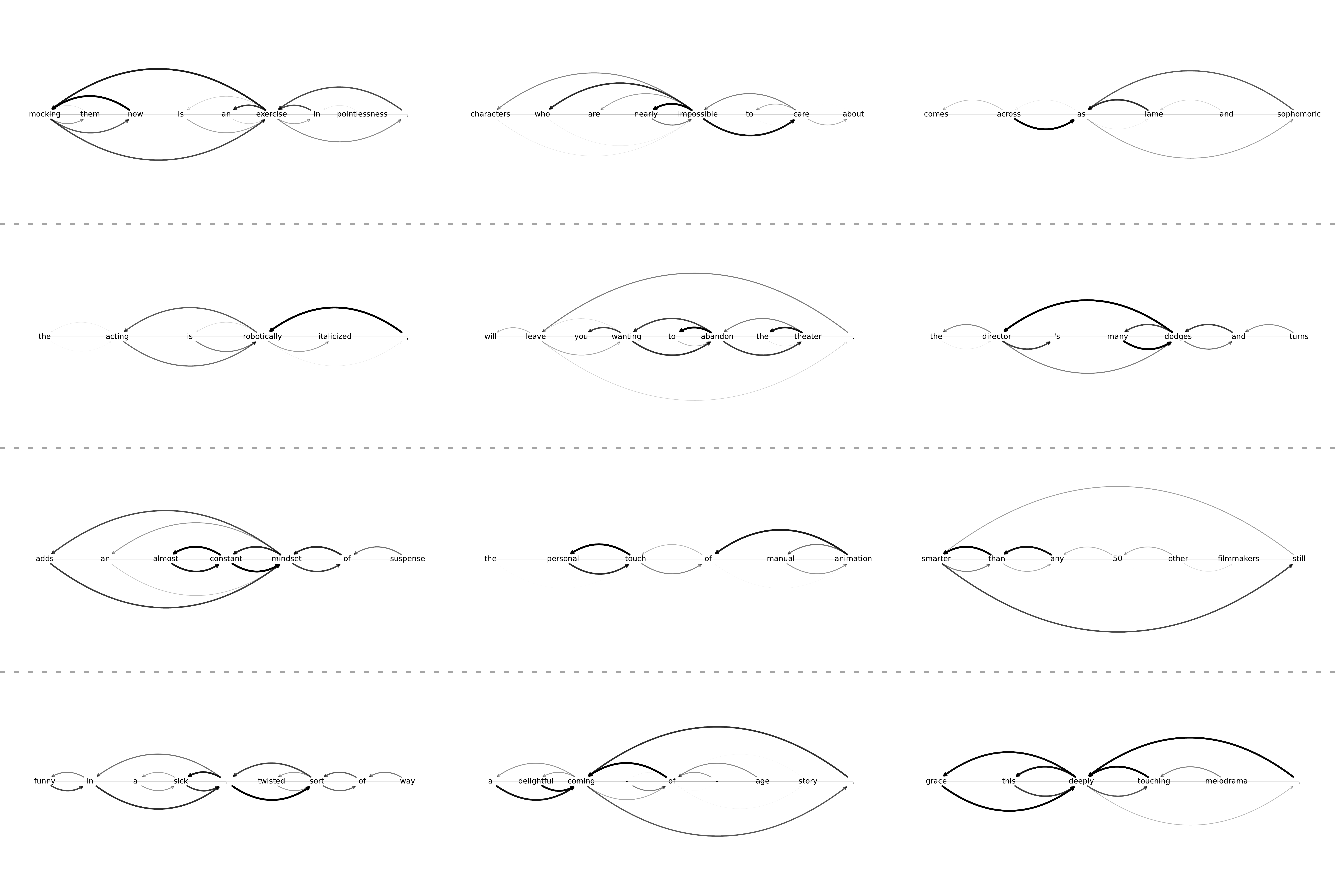}}
\caption{Visualizing label-relevant subgraphs discovered by \proj for Graph-SST2. The top two rows show sentences with negative sentiment, and the bottom two rows show sentences with positive sentiment.}
\end{center}
\end{figure}

\begin{figure}[t]
\begin{center}
\centerline{\includegraphics[width=1.0\columnwidth]{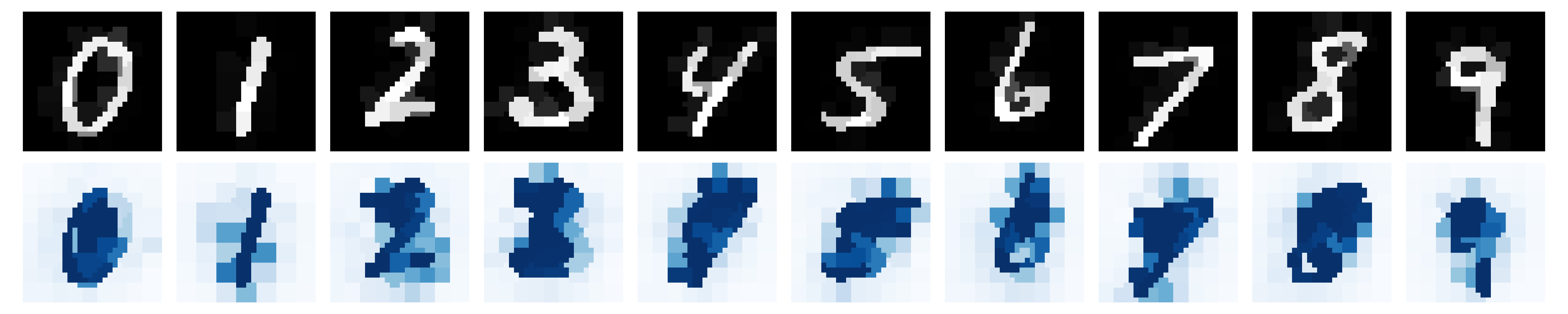}}
\caption{Visualizing label-relevant subgraphs discovered by \proj for MNIST-75sp. The first row shows the raw images and the second row shows the normalized attention weights learned by \proj.}
\label{viz:mnist}
\end{center}
\end{figure}


\end{document}